\documentclass{article}
\usepackage[compress,sort,numbers]{natbib}
\usepackage{microtype}
\usepackage{graphicx}
\usepackage{subfigure}
\usepackage{booktabs} %
\usepackage{hyperref}
\usepackage{enumitem}

\usepackage[accepted]{icml2022} %

\usepackage{amsmath, amssymb,amsthm,xspace,bbm,parskip}
\usepackage{thm-restate}

\theoremstyle{plain}
\newtheorem{theorem}{Theorem}[section]

\newtheorem{lemma}[theorem]{Lemma}
\newtheorem{corollary}[theorem]{Corollary}
\theoremstyle{definition}
\newtheorem{definition}[theorem]{Definition}

\theoremstyle{remark}

\newcommand{\inprod}[2]{\left\langle #1, #2 \right\rangle}
\DeclareMathOperator*{\hull}{hull}
\DeclareMathOperator*{\cone}{cone}
\DeclareMathOperator*{\sort}{sort}
\DeclareMathOperator*{\argmax}{arg\,max}
\DeclareMathOperator*{\argmin}{arg\,min}
\DeclareMathOperator*{\argsup}{arg\,sup}
\DeclareMathOperator*{\Risk}{Risk}
\DeclareMathOperator*{\Regret}{Regret}

\usepackage[colorinlistoftodos,textsize=tiny]{todonotes} %
\usepackage{accents}
\usepackage{bbm}
\usepackage{xspace}

\usepackage{tikz,pgfplots,tikz-3dplot}
\usetikzlibrary{calc}
\usetikzlibrary{arrows.meta,shapes.misc,positioning,fit,calc,intersections,through,decorations.markings}
\usetikzlibrary{backgrounds}
\tdplotsetmaincoords{55}{135}        %
\pgfplotsset{compat=1.17}

\tikzstyle{cell}=[dashed,thick]
\tikzstyle{simplex}=[thick]
\newcommand{\tikzfigscale}{2.2}

\usetikzlibrary{arrows}
\usetikzlibrary{shapes}

\newcommand{\restatableeq}[3]{\label{#3}#2\gdef#1{#2\tag{\ref{#3}}}}

\newcommand{\E}{\mathbb{E}}

\renewcommand{\P}{\mathcal{P}}
\newcommand{\R}{\mathcal{R}}
\newcommand{\Sc}{\mathcal{S}}
\newcommand{\U}{\mathcal{U}}

\newcommand{\Y}{\mathcal{Y}}
\newcommand{\T}{\mathcal{T}}

\newcommand{\Rlo}{R^{\tt low}}
\newcommand{\Rhi}{R^{\tt high}}
\newcommand{\simplex}{{\Delta_{\Y}}}
\newcommand{\lk}{\ell_k}
\newcommand{\Lk}{L_k}

\newcommand{\Tk}{T_k}
\newcommand{\Rk}{\R_k}
\newcommand{\sumk}{\sigma_k}
\newcommand{\topkset}{T_k}
\newcommand{\topkvec}{\tau_k}
\newcommand{\link}{\psi_k}
\newcommand{\propk}{\Gamma_k}
\newcommand{\ones}{\mathbbm{1}}
\newcommand{\prop}[1]{\mathrm{prop}[#1]}
\newcommand{\reals}{\mathbb{R}}
\newcommand{\relint}{\mathrm{relint}}
\newcommand{\risk}[1]{\underline{#1}}
\newcommand{\support}{\mathrm{support}}
\newcommand{\sign}{\mathrm{sign}}

\newcommand{\toto}{\rightrightarrows}
\renewcommand{\i}{{[i]}}
\renewcommand{\j}{{[j]}}
\renewcommand{\k}{{[k]}}
\newcommand{\kp}{{[k+1]}}
\newcommand{\sm}[1]{{\setminus{#1}}}
\newcommand{\smy}{\sm{y}}
\newcommand{\Li}[1]{L^{(#1)}}
\newcommand{\Ri}[1]{\R^{(#1)}}
\newcommand{\elli}[1]{\ell^{(#1)}}
\newcommand{\Upostwo}{U_+^{(2)}}

\icmltitlerunning{Consistent Polyhedral Surrogates for Top-$k$ Classification and Variants}

\begin{document}
\twocolumn[
\icmltitle{Consistent Polyhedral Surrogates for Top-$k$ Classification and Variants}
\icmlsetsymbol{equal}{*}

\begin{icmlauthorlist}
\icmlauthor{Jessie Finocchiaro}{cu}
\icmlauthor{Rafael Frongillo}{cu}
\icmlauthor{Emma Goodwill}{cu}
\icmlauthor{Anish Thilagar}{cu}
\end{icmlauthorlist}

\icmlaffiliation{cu}{University of Colorado Boulder Department of Computer Science, Boulder, CO, USA}

\icmlcorrespondingauthor{Jessie Finocchiaro}{Jessica.Finocchiaro@colorado.edu}
\icmlcorrespondingauthor{Anish Thilagar}{anish@colorado.edu}
\icmlkeywords{Top-$k$ classification, surrogate loss design}

\vskip 0.3in]

\printAffiliationsAndNotice{}

\begin{abstract}
Top-$k$ classification is a generalization of multiclass classification used widely in information retrieval, image classification, and other extreme classification settings. 
Several hinge-like (piecewise-linear) surrogates have been proposed for the problem, yet all are either non-convex or inconsistent.
For the proposed hinge-like surrogates that are convex (i.e., polyhedral), we apply the recent embedding framework of \citet{finocchiaro2019embedding,finocchiaro2022embedding} to determine the prediction problem for which the surrogate is consistent.
These problems can all be interpreted as variants of top-$k$ classification, which may be better aligned with some applications.
We leverage this analysis to derive constraints on the conditional label distributions under which these proposed surrogates become consistent for top-$k$.
It has been further suggested that every convex hinge-like surrogate must be inconsistent for top-$k$.
Yet, we use the same embedding framework to give the first consistent polyhedral surrogate for this problem.
\end{abstract}

\section{Introduction}

Top-$k$ classification is commonly used in image recognition~\citep{russakovsky2015imagenet,Karpathy2014large,akata2013good} and action analysis~\citep{Furnari2018leveraging}, search querying~\citep{ailon2008efficient,reddi2019stochastic}, and recommender systems more broadly~\citep{adomavicius2016classification,billsus1998learning,deshpande2004item}.
For example, in information retrieval, a page of $k$ results may be displayed out of $n \gg k$ total webpages available, with success indicated by a user clicking one of these $k$.
This scenario can be captured by the \emph{top-$k$ loss}: given a set $S$ of labels, $|S|=k$, and the true label $y$, assign loss $1$ if $y \not \in S$, and $0$ otherwise.
As top-$k$ loss is discrete, it is typically computationally hard to optimize.
Therefore, top-$k$ learning algorithms typically employ a \emph{surrogate loss}.

Common desiderata for surrogate losses are that they be \emph{convex}, and thus easier to optimize, and that they be \emph{statistically consistent}, meaning they solve the original problem (here: top-$k$) when given enough data.
Another consideration is whether the surrogate is \emph{smooth (e.g.\ differentiable) or piecewise-linear (``hinge-like'')}.
This consideration is related to whether the surrogate will implicitly learn the underlying conditional label distribution, which generally is a harder learning problem than the original; for example, the entire label distribution contains more information than the set of $k$ most likely labels.
Typically, smooth surrogates, such as cross-entropy, implicitly learn the entire label distribution.%
\footnote{Concretely, consider any surrogate whose Bayes risk is strictly concave, which is the case for most smooth surrogates.
For each surrogate prediction $u$, it can minimize expected loss for at most one conditional label distribution $p$; otherwise the Bayes risk would be flat on the line segment between two such distributions.
Thus, one can infer $p$ from the $u$ returned by the model.}
Conventional wisdom has been that piecewise-linear surrogates are more ``efficient'' in the sense that they learn only what is relevant for the original problem.
Moreover, piecewise-linear and convex surrogates give rise to linear surrogate regret bounds, whereas most smooth surrogates do not~\cite{frongillo2021surrogate}.

Combining the above desiderata, we would like a surrogate which is both \emph{polyhedral} (convex and piecewise-linear) and consistent for top-$k$ classification.
Unfortunately, while many piecewise-linear surrogates have been proposed for top-$k$, they are all either non-convex or inconsistent~\citep{lapin2015top,lapin2016loss,lapin2018analysis,yang2018consistency,reddi2019stochastic}.
Moreover, the results and writing of both \citet[pg.6]{lapin2016loss} and \citet[pg.1]{yang2018consistency} suggest that perhaps no such surrogate exists for top-$k$.

We resolve this open question by presenting the first consistent polyhedral surrogate for top-$k$ classification (\S~\ref{sec:embedding-construction}).
Our proof uses embedding framework of \citet{finocchiaro2019embedding,finocchiaro2022embedding}.
We also use the embedding framework to analyze three previous polyhedral surrogates in the literature which are inconsistent for top-$k$ (\S~\ref{sec:previous-surrogates}).
For each we show (a) what discrete prediction problem the surrogate is actually solving, in all cases a natural variant of top-$k$, and (b) a constraint on the conditional label distributions such that the surrogate becomes consistent for top-$k$.
Finally, we evaluate the performance of our surrogate compared to these previous surrogates (\S~\ref{sec:regret-comparison}).

\section{Setting}

We consider predictions in a discrete set $\R$ over a finite set of labels $\Y = \{1, \ldots, n\}$, and conditional label distributions $\simplex$.
In top-$k$ classification, predictions take the form of size-$k$ subsets of labels, $\R = \Rk := \{ S \subseteq \Y \mid |S| = k \}$. 
Top-$k$ loss $\lk:\Rk \times \Y \to \reals_+$ simply tests whether the actual label lies in the set,
\begin{align}\label{eq:top-k}
    \lk(S,y) &= \ones \{y \not \in S\}~,
\end{align}
where $\ones\{E\}$ is $1$ if event $E$ is true, and $0$ otherwise.
In reasoning about top-$k$ and variants, it is often useful to denote $u_{[i]}$ to be the $i^{\text{th}}$ largest element of the vector $u \in \reals^n$.
Moreover, the set of possible top-$k$ indices $\topkset:\reals^n\to 2^{\Rk}$ is given by $\topkset: u \mapsto \argmax_{S \in\Rk} \inprod{\ones_S}{u}$.
Observe $|\topkset(u)| > 1$ if and only if $u_\k = u_\kp$.
Additionally, we denote the sum of these top $k$ elements by $\sigma_k(u) = \max_{S \in \Rk} \inprod{\ones_S}{u}$.

\subsection{Consistency, Property Elicitation, and Calibration}
Discrete losses such as $\lk$ are hard to optimize directly, so a consistent surrogate is sought instead with better optimization guarantees.
In essence, a surrogate and link are \emph{consistent} with respect to a discrete target loss if approaching the optimal surrogate loss implies approaching the optimal target loss when the link function is applied to the surrogate predictions. 
We will phrase consistency in terms of the equivalent notion of calibration~\citep{bartlett2008classification,ramaswamy2016convex,tewari2007consistency,steinwart2008support}.

Before defining calibration, we first introduce properties, which encode the optimal predictions for a loss as a function of the conditional label distribution.
Here $\P\subseteq\simplex$.
\begin{definition} \label{def:property}
    A \emph{property} is a function $\Gamma : \P \to 2^\R \setminus \{ \emptyset \}$, which we more succinctly denote $\Gamma : \P \toto \R$.
  A loss $L : \R \times \Y \to \reals$ \emph{elicits} a property $\Gamma : \P \toto \R$ if
  \begin{align*}
      \forall p \in \P, \quad \Gamma(p) = \argmin_{r \in \R} \E_{Y \sim p} L(r,Y)~.~
  \end{align*}
\end{definition}
A loss $L$ is \emph{minimizable} if $\E_{Y \sim p} L(\cdot, Y)$ attains its infimum for all $p \in \P$.
Every minimizable loss $L$ elicits a unique property, which we denote $\prop{L}$.

As an example, the property elicited by top-$k$ loss is $\gamma_k = \prop{\lk}$, which is given by
    \begin{align}
        \gamma_k(p)
        & = \arg\min_{S \in \Rk}\inprod{p}{\lk(S, \cdot)} \nonumber \\
        &= \arg\min_{S \in \Rk} \sum_{i \not \in S} p_i \nonumber \\
        & = \topkset|_{\simplex}(p) ~. \label{eqn:topk-property}
    \end{align}

\begin{definition}\label{def:calibration}
    Let $\ell:\R\times\Y\to\reals$ with $|\R|<\infty$.
  A surrogate $L : \reals^d \times \Y \to \reals_+$ and link $\psi : \reals^d \to \R$ pair $(L, \psi)$ is \emph{calibrated} with respect to $\ell$ over $\P\subseteq\simplex$ if for all $p\in\P$,
  \begin{align*}
      \inf_{u : \psi(u) \not \in \prop{\ell}(p)} \E_{Y \sim p} L(u,Y) > \inf_{u \in \reals^d} \E_{Y \sim p} L(u,Y)~.
  \end{align*}
  We simply say $L$ is calibrated with respect to $\ell$ if there exists a link $\psi$ such that $(L,\psi)$ is calibrated with respect to $\ell$.
\end{definition}
One can think of $\P$ as the set of possible conditional label distributions conditioned on some feature vector.
We consider $\P = \simplex$ unless otherwise specified.

\subsection{Embedding Framework for Polyhedral Surrogates}

We rely heavily on the embedding framework of \citet{finocchiaro2019embedding,finocchiaro2022embedding}, which gives tools to analyze and construct consistent polyhedral surrogates.
An embedding maps the finite set of target predictions to a \emph{representative set} of surrogate predictions.

\begin{definition}[Representative set]
  A set $\Sc \subseteq \R$ is \emph{representative} for a property $\Gamma :\P \toto \R$ if, for all $p \in \P$, we have $\Gamma(p) \cap \Sc \neq \emptyset$.
  We say $\Sc$ is representative for a loss $L$ if it is representative for the property $\prop{L}$.
\end{definition}

\begin{definition}[Embedding]\label{def:loss-embed}
  A loss $L:\reals^d\times \Y\to\reals_+$ \emph{embeds} a discrete loss $\ell:\R\times \Y\to\reals_+$ if there exists a representative set $\Sc$ for $\ell$ and an injective embedding $\varphi:\Sc\to\reals^d$ such that
  (i) for all $r\in\Sc$ and $y \in \Y$ we have $L(\varphi(r),y) = \ell(r,y)$, and (ii) for all $p\in\simplex,r\in\Sc$ we have
  \begin{equation}\label{eq:embed-loss}
    r \in \prop{\ell}(p) \iff \varphi(r) \in \prop{L}(p)~.
  \end{equation}
\end{definition}
In other words, a surrogate embeds a discrete target loss if the loss values match at the embedded points, and moreover, a target prediction is optimal exactly when its embedded prediction is optimal for the surrogate.

Embeddings are closely tied to polyhedral surrogates; in particular, every polyhedral surrogate embeds some discrete loss \citep{finocchiaro2022embedding}.
We will primarily use the following results.
Throughout, for a loss $L : \R \times \Y \to \reals_+$ and set $\Sc \subseteq \R$, we denote by $L|_{\Sc}$ the loss on $\Sc \times \Y$ given by $L|_\Sc(u,y) = L(u,y)$, i.e., the restriction of $L$ to $\Sc$.
\begin{theorem}[{\citet{finocchiaro2022embedding}}]\label{thm:polyhedral-embeds-discrete}

~%
\begin{enumerate}
    \item Every polyhedral loss $L$ has a finite representative set.
    \item If $\Sc$ is a finite representative set for $L$, then $L$ embeds the discrete loss $L|_{\Sc}$.
    \item If $L$ embeds $\ell$, then there exists a link $\psi$ such that $(L,\psi)$ is calibrated with respect to $\ell$.
\end{enumerate}
\end{theorem}
These correspond to Lemma 2, Proposition 1, and Theorem 2 in that work, respectively.
The authors also provide a construction for the calibrated link $\psi$, as well as a construction for a calibrated polyhedral surrogate given any discrete loss; we discuss both of these additional tools in \S~\ref{sec:embedding-construction}.

\section{Previous Polyhedral Surrogates}\label{sec:previous-surrogates}

\citet{lapin2015top} proposes a nonconvex surrogate for top-$k$ prediction, as well as convex upper bounds on this surrogate in~\citep{lapin2016loss}, denoted $\Li{2}$ and $\Li{3}$ here to parallel their notation.
\citet{yang2018consistency} show that $\Li{2}$ and $\Li{3}$ are inconsistent for $\lk$ classification, and introduce another inconsistent surrogate, which we denote $\Li{4}$.

All three losses $\Li{2}$, $\Li{3}$, and $\Li{4}$ are polyhedral; as such Theorem~\ref{thm:polyhedral-embeds-discrete} implies that they all embed \emph{some} discrete loss.
It is not immediately clear, however, what exactly these discrete losses are for each surrogate.
In this section, we derive a target loss that each surrogate embeds, which in each case is an interesting variant of the original top-$k$ problem.

Deriving the loss embedded by an inconsistent surrogate also allows one to understand when it would be consistent for the intended target.
In particular, by looking at the geometry of the property elicited by the surrogate, we can derive a constraint on the set of conditional label distributions under which it becomes consistent for top-$k$.
One can view these results as a refinement of inconsistency results; for example, \citet[Proposition 4.2]{yang2018consistency} characterizes the set of distributions such that the surrogate report $u = \vec 0 \in\reals^n$ is optimal, a subset of the set of distributions we eliminate.

In summary, then, we strive in this section to answer two questions about $\Li{2}$, $\Li{3}$, and $\Li{4}$: (i) What discrete loss does the surrogate embed?  (ii) On which conditional label distributions is the surrogate actually consistent for top-$k$?

To answer (i), we find a finite representative set and apply Theorem~\ref{thm:polyhedral-embeds-discrete}, which shows that restricting to that set gives an embedding.
To find this set, we first observe that these surrogates are all invariant in the $\ones$ direction, meaning $L(u,y) = L(u + \alpha \ones, y)$ for all $\alpha \in \reals$.
Furthermore, we can fix the lowest $n-k-1$ elements of $U$ to be the same as $u_\kp$, as this can only improve the loss on any outcome.
We can therefore restrict our attention to the set of reports
\begin{equation}\label{eq:U-inf-rep}
    U = \{u \in\reals^n_+ \mid u_\kp = 0 = u_{[n]} \},
\end{equation}
which is representative, although infinite.
In some cases, we further restrict $U$ to a region where the positive part operator $(\cdot)_+$ can be removed.
In each case, we partition the resulting set into polytope regions over which the surrogate is affine; in other words, we find the pieces for which the loss is piecewise linear.
By the theory of polyhedral functions, for each conditional label distribution, at least one vertex of one of these regions must be a minimizer of the expected loss.
The union of all such vertices therefore yields a finite representative set.
As a final step, in each case we reparameterize this set of vertices with a bijection to a more natural prediction set, which more transparently reveals a variant of the top-$k$ problem.
Applying such a bijection preserves the embedding by Definition~\ref{def:loss-embed}.

To answer (ii), we observe that in all cases, inconsistency is driven by surrogate reports for which the set of top-$k$ elements is ambiguous, thus forcing the link to break a tie.
Specifically, for reports $u\in\reals^n$ with $u_\k = u_\kp$, we have multiple options for $\topkset(u)$, yet $\link$ must select one.
Let $U_{\text{ambig}} = \{ u\in\reals^n \mid u_\k = u_\kp \}$ be the set of these ambiguous surrogate reports.
Whenever a report $u \in U_{\text{ambig}}$ is optimal for a conditional label distribution $p$ for which $\topkset(p)$ is \emph{not} ambiguous, i.e.\ $p_\k > p_\kp$, we will have inconsistency.
Therefore, $(\Li{i},\link)$ is consistent with respect to $\lk$ on the set $\P^{(i)} := \{ p \in \simplex \mid \prop{\Li{i}}(p) \cap U_{\text{ambig}} = \emptyset \}$ of conditional label distributions for which there is no ambiguous optimal report.

\begin{table*}
\centering
\begin{tabular}{ccccc}
    & $\Li{2}$ & $\Li{3}$ & $\Li{4}$ & $\Lk$\\
    \hline \hline
    \rotatebox[origin=c]{90}{$k = 2$\hspace*{-2.5cm}} & \begin{tikzpicture} [scale=\tikzfigscale, thick, tdplot_main_coords]
\coordinate (orig) at (0,0,0);

\coordinate (uniform) at (1/4,1/4,1/4);
\coordinate[label=below:${p_1}$] (h) at (3/4,0,0);
\coordinate[label=below:${p_3}$] (m) at (0,3/4,0);
\coordinate[label=right:${p_2}$] (l) at (0,0,3/4);

\coordinate (one) at (3/4 - 1/8, 0, 1/8);
\coordinate (two) at (3/4 - 1/8, 1/8, 0);
\coordinate (three) at (1/8, 3/4 - 1/8, 0);
\coordinate (four) at (0, 3/4 - 1/8, 1/8);
\coordinate (five) at (0, 1/8, 3/4-1/8);
\coordinate (six) at (1/8, 0, 3/4-1/8);
\coordinate (seven) at (1/3, 0, 3/4 - 1/3);
\coordinate (eight) at (3/4 - 1/3, 0, 1/3);
\coordinate (nine) at (3/4 - 1/3, 1/3, 0);
\coordinate (ten) at (1/3, 3/4 - 1/3, 0);
\coordinate (eleven) at (0, 3/4 - 1/3, 1/3);
\coordinate (twelve) at (0, 1/3, 3/4 - 1/3);
\coordinate (thirteen) at (1/3, 3/4 - 2/3, 1/3);
\coordinate (fourteen) at (1/3, 1/3, 3/4 - 2/3);
\coordinate (fifteen) at (3/4 - 2/3, 1/3, 1/3);

\draw[simplex] (h) -- (m) -- (l) -- (h);

\begin{scope}
\clip (h) -- (m) -- (l) -- (h);

\draw (h) -- (one) -- (two) -- cycle;
\draw (l) -- (five) -- (six) -- cycle;
\draw (m) -- (three) -- (four) -- cycle;

\draw[fill = blue, fill opacity = 0.1] (one) -- (two) -- (nine)-- (fourteen) -- (thirteen) -- (eight) -- cycle;
\draw[fill = blue, fill opacity = 0.1] (three) -- (four) -- (eleven) -- (fifteen) -- (fourteen) -- (ten) -- cycle;
\draw[fill = blue, fill opacity = 0.1] (five) -- (six) -- (seven) -- (thirteen) -- (fifteen) -- (twelve) -- cycle;

\draw (nine) -- (ten) -- (fourteen) -- cycle;
\draw (seven) -- (eight) -- (thirteen) -- cycle;
\draw (eleven) -- (twelve) -- (fifteen) -- cycle;

\draw[fill = blue, fill opacity = 0.1] (thirteen) -- (fourteen) -- (fifteen) -- cycle;

\draw[dashed, color=blue, opacity = 0.6] (1/2, 1/4, 0) -- (uniform);
\draw[dashed, color=blue, opacity = 0.6] (1/2, 0, 1/4) -- (uniform);
\draw[dashed, color=blue, opacity = 0.6](1/4, 0, 1/2) -- (uniform);
\draw[dashed, color=blue, opacity = 0.6] (0, 1/4, 1/2) -- (uniform);
\draw[dashed, color=blue, opacity = 0.6](0, 1/2, 1/4) -- (uniform);
\draw[dashed, color=blue, opacity = 0.6] (1/4, 1/2, 0) -- (uniform);

\end{scope}

\node (129138center) at (23/48, 13/96, 13/96) {};
\node (129138label) at (23/48, 13/96, -1/5) {{\tiny $(\emptyset, 1)$}};
\draw[-{Latex[length=1mm, width=1mm]}, red] (129138label) -- (129138center.center);

\node (91014center) at (3/8 - 1/48, 3/8 - 1/48, 1/24) {};
\node (91014label) at (3/8, 3/8, -1/5) {{\tiny $(\emptyset, 13)$}};
\draw[-{Latex[length=1mm, width=1mm]}, red] (91014label) -- (91014center.center);

\node (midtricenter) at (1/6, 3/8 - 1/12, 3/8 - 1/12) {};
\node[rotate=0] (midtrilabel) at (-1/6, 3/8, 3/8) {{\tiny $(\emptyset, \emptyset)$}};
\draw[-{Latex[length=1mm, width=1mm]}, red] (midtrilabel) -- (midtricenter.center);

\node (endtricenter) at (3/4 - 1/12, 1/24, 1/24) {};
\node (endtrilabel) at (3/4-1/6, -1/6, 1/6) {{\tiny $(1,4)$}};
\draw[-{Latex[length=1mm, width=1mm]}, red] (endtrilabel) -- (endtricenter.center);

\end{tikzpicture} & \begin{tikzpicture} [scale=\tikzfigscale, thick, tdplot_main_coords]
\coordinate (orig) at (0,0,0);

\coordinate (uniform) at (1/4,1/4,1/4);
\coordinate[label=below:${p_1}$] (h) at (3/4,0,0);
\coordinate[label=below:${p_3}$] (m) at (0,3/4,0);
\coordinate[label=right:${p_2}$] (l) at (0,0,3/4);

\coordinate (one) at (3/4 - 1/8, 0, 1/8);
\coordinate (two) at (3/4 - 1/8, 1/8, 0);
\coordinate (three) at (1/8, 3/4 - 1/8, 0);
\coordinate (four) at (0, 3/4 - 1/8, 1/8);
\coordinate (five) at (0, 1/8, 3/4-1/8);
\coordinate (six) at (1/8, 0, 3/4-1/8);
\coordinate (seven) at (1/3, 0, 3/4 - 1/3);
\coordinate (eight) at (3/4 - 1/3, 0, 1/3);
\coordinate (nine) at (3/4 - 1/3, 1/3, 0);
\coordinate (ten) at (1/3, 3/4 - 1/3, 0);
\coordinate (eleven) at (0, 3/4 - 1/3, 1/3);
\coordinate (twelve) at (0, 1/3, 3/4 - 1/3);
\coordinate (thirteen) at (1/3, 3/4 - 2/3, 1/3);
\coordinate (fourteen) at (1/3, 1/3, 3/4 - 2/3);
\coordinate (fifteen) at (3/4 - 2/3, 1/3, 1/3);

\draw[simplex] (h) -- (m) -- (l) -- (h);

\begin{scope}
\clip (h) -- (m) -- (l) -- (h);

\draw (h) -- (one) -- (two) -- cycle;
\draw (l) -- (five) -- (six) -- cycle;
\draw (m) -- (three) -- (four) -- cycle;

\draw[fill = blue, fill opacity = 0.1] (one) -- (two) -- (nine)-- (fourteen) -- (thirteen) -- (eight) -- cycle;
\draw[fill = blue, fill opacity = 0.1] (three) -- (four) -- (eleven) -- (fifteen) -- (fourteen) -- (ten) -- cycle;
\draw[fill = blue, fill opacity = 0.1] (five) -- (six) -- (seven) -- (thirteen) -- (fifteen) -- (twelve) -- cycle;

\draw (nine) -- (ten) -- (fourteen) -- cycle;
\draw (seven) -- (eight) -- (thirteen) -- cycle;
\draw (eleven) -- (twelve) -- (fifteen) -- cycle;

\draw[fill = blue, fill opacity = 0.1] (thirteen) -- (fourteen) -- (fifteen) -- cycle;

\draw[dashed, color=blue, opacity = 0.6] (1/2, 1/4, 0) -- (uniform);
\draw[dashed, color=blue, opacity = 0.6] (1/2, 0, 1/4) -- (uniform);
\draw[dashed, color=blue, opacity = 0.6](1/4, 0, 1/2) -- (uniform);
\draw[dashed, color=blue, opacity = 0.6] (0, 1/4, 1/2) -- (uniform);
\draw[dashed, color=blue, opacity = 0.6](0, 1/2, 1/4) -- (uniform);
\draw[dashed, color=blue, opacity = 0.6] (1/4, 1/2, 0) -- (uniform);

\end{scope}

\node (129138center) at (23/48, 13/96, 13/96) {};
\node (129138label) at (23/48, 13/96, -1/6) {{\tiny $(234,1)$}};
\draw[-{Latex[length=1mm, width=1mm]}, red] (129138label) -- (129138center.center);

\node (91014center) at (3/8 - 1/48, 3/8 - 1/48, 1/24) {};
\node (91014label) at (3/8, 3/8, -1/4) {{\tiny $(24,13)$}};
\draw[-{Latex[length=1mm, width=1mm]}, red] (91014label) -- (91014center.center);

\node (midtricenter) at (1/6, 3/8 - 1/12, 3/8 - 1/12) {};
\node[rotate=0] (midtrilabel) at (-1/5, 3/8, 3/8) {{\tiny $(1234)$}};
\draw[-{Latex[length=1mm, width=1mm]}, red] (midtrilabel) -- (midtricenter.center);

\node (endtricenter) at (3/4 - 1/12, 1/24, 1/24) {};
\node[rotate = 60] (endtrilabel) at (3/4-1/6, -1/6, 1/6) {{\tiny $(23,4,1)$}};
\draw[-{Latex[length=1mm, width=1mm]}, red] (endtrilabel) -- (endtricenter.center);

\end{tikzpicture} & \begin{tikzpicture} [scale=\tikzfigscale, thick, tdplot_main_coords]
\coordinate (orig) at (0,0,0);

\coordinate (uniform) at (1/4,1/4,1/4);
\coordinate[label=below:${p_1}$] (h) at (3/4,0,0);
\coordinate[label=below:${p_3}$] (m) at (0,3/4,0);
\coordinate[label=right:${p_2}$] (l) at (0,0,3/4);

\coordinate (one) at (1/2, 1/4, 0);
\coordinate (two) at (1/4, 1/2, 0);
\coordinate (three) at (0, 1/2, 1/4);
\coordinate (four) at (0, 1/4, 1/2);
\coordinate (five) at (1/4, 0, 1/2);
\coordinate (six) at (1/2, 0, 1/4);
\coordinate (seven) at (1/3, 1/3, 1/12);
\coordinate (eight) at (1/12, 1/3, 1/3);
\coordinate (nine) at (1/3, 1/12, 1/3);

\draw[simplex] (h) -- (m) -- (l) -- (h);

\begin{scope}
\clip (h) -- (m) -- (l) -- (h);

\draw (h) -- (one) -- (six) -- cycle;
\draw (l) -- (four) -- (five) -- cycle;
\draw (m) -- (three) -- (two) -- cycle;

\draw (one) -- (two) -- (seven) -- cycle;
\draw (three) -- (four) -- (eight) -- cycle;
\draw (five) -- (six) -- (nine) -- cycle;

\draw[fill=blue,fill opacity = 0.1] (one) -- (seven) -- (nine) -- (six) -- cycle;
\draw[fill=blue,fill opacity = 0.1] (two) -- (three) -- (eight) -- (seven) -- cycle;
\draw[fill=blue,fill opacity = 0.1] (four) -- (five) -- (nine) -- (eight) -- cycle;

\draw[fill=blue,fill opacity = 0.1] (seven) -- (eight) -- (nine) -- cycle;

\draw[dashed, color=blue, opacity = 0.6] (1/2, 1/4, 0) -- (uniform);
\draw[dashed, color=blue, opacity = 0.6] (1/2, 0, 1/4) -- (uniform);
\draw[dashed, color=blue, opacity = 0.6](1/4, 0, 1/2) -- (uniform);
\draw[dashed, color=blue, opacity = 0.6] (0, 1/4, 1/2) -- (uniform);
\draw[dashed, color=blue, opacity = 0.6](0, 1/2, 1/4) -- (uniform);
\draw[dashed, color=blue, opacity = 0.6] (1/4, 1/2, 0) -- (uniform);

\end{scope}

\node (13center) at (3/8 - 1/48, 3/8 - 1/48, 1/24) {};
\node (13label) at (3/8, 3/8, -1/5) {{\tiny $1,3$}};
\draw[-{Latex[length=1mm, width=1mm]}, red] (13label) -- (13center.center);

\node (emptycenter) at (.23 * .75, .38 * .75,.38 * .75) {};
\node (emptylabel) at (-1/8, 3/8, 3/8) {{\tiny $\emptyset$}};
\draw[-{Latex[length=1mm, width=1mm]}, red] (emptylabel) -- (emptycenter.center);

\node (2center) at (.2 * .75, .2 * .75,.6 * .75) {};
\node (2label) at (3/8, -1/8, 5/8) {{\tiny $2$}};
\draw[-{Latex[length=1mm, width=1mm]}, red] (2label) -- (2center.center);

\node (14center) at (.8 * .75, .1 * .75,.1 * .75) {};
\node (14label) at (5/8, -1/8, 3/8) {{\tiny $14$}};
\draw[-{Latex[length=1mm, width=1mm]}, red] (14label) -- (14center.center);

\end{tikzpicture} & \begin{tikzpicture} [scale=\tikzfigscale, thick, tdplot_main_coords]
\coordinate (orig) at (0,0,0);

\coordinate (uniform) at (1/4,1/4,1/4);
\coordinate[label=below:${p_1}$] (h) at (3/4,0,0);
\coordinate[label=below:${p_3}$] (m) at (0,3/4,0);
\coordinate[label=right:${p_2}$] (l) at (0,0,3/4);

\draw[simplex] (h) -- (m) -- (l) -- (h);

\begin{scope}
\clip (h) -- (m) -- (l) -- (h);

\draw[opacity=0.9] (2/3, 1/3, 0) -- (uniform);
\draw[opacity=0.9] (2/3, 0, 1/3) -- (uniform);
\draw[opacity=0.9] (1/3, 0, 2/3) -- (uniform);
\draw[opacity=0.9] (0, 1/3, 2/3) -- (uniform);
\draw[opacity=0.9] (0, 2/3, 1/3) -- (uniform);
\draw[opacity=0.9] (1/3, 2/3, 0) -- (uniform);

\end{scope}

\node (14center) at (.66 * .75, 0.16 * .75, 0.16 * 0.75) {};
\node (14label) at (5/8, -1/4,2/8) {{\tiny $14$}};
\draw[-{Latex[length=1mm, width=1mm]}, red] (14label) -- (14center.center);

\node (23center) at (.1 * .75, 0.45 * .75, 0.45 * 0.75) {};
\node (23label) at (-1/8, 1/2,1/2) {{\tiny $23$}};
\draw[-{Latex[length=1mm, width=1mm]}, red] (23label) -- (23center.center);

\end{tikzpicture}\\ 
    \rotatebox[origin=c]{90}{$k = 3$\hspace*{-2.5cm}} & \begin{tikzpicture} [scale=\tikzfigscale, thick, tdplot_main_coords]
\coordinate (orig) at (0,0,0);

\coordinate (uniform) at (1/4,1/4,1/4);
\coordinate[label=below:${p_1}$] (h) at (3/4,0,0);
\coordinate[label=below:${p_3}$] (m) at (0,3/4,0);
\coordinate[label=right:${p_2}$] (l) at (0,0,3/4);

\coordinate (one) at (1/2, 1/20, 3/4 - 1/2 - 1/20);
\coordinate (two) at (1/2, 3/4 - 1/2 - 1/20,1/20);
\coordinate (three) at (3/4 - 1/2 - 1/20,1/2,1/20);
\coordinate (four) at (1/20,1/2, 3/4 - 1/2 - 1/20);
\coordinate (five) at (1/20,3/4 - 1/2 - 1/20, 1/2);
\coordinate (six) at (3/4 - 1/2 - 1/20,1/20, 1/2);
\coordinate (seven) at (1/2, 0, 1/4);
\coordinate (eight) at (1/2, 1/4, 0);
\coordinate (nine) at (1/4, 1/2, 0);
\coordinate (ten) at (0, 1/2, 1/4);
\coordinate (eleven) at (0, 1/4, 1/2);
\coordinate (twelve) at (1/4, 0, 1/2);
\coordinate (thirteen) at (1/3, 1/12, 1/3);
\coordinate (fourteen) at (1/3, 1/3, 1/12);
\coordinate (fifteen) at (1/12, 1/3, 1/3);
\coordinate (sixteen) at (1/6, 1/4, 1/3);
\coordinate (seventeen) at (1/4, 1/6, 1/3);
\coordinate (eighteen) at (1/3, 1/6, 1/4);
\coordinate (nineteen) at (1/3, 1/4, 1/6);
\coordinate (twenty) at (1/4, 1/3, 1/6);
\coordinate (twentyone) at (1/6, 1/3, 1/4);
\coordinate(twentytwo) at (1/4, 1/4, 1/4);

\draw[simplex] (h) -- (m) -- (l) -- (h);

\begin{scope}
\clip (h) -- (m) -- (l) -- (h);

\draw[fill=blue,fill opacity = 0.1] (h) -- (one) -- (two) -- cycle;
\draw[fill=blue,fill opacity = 0.1] (m) -- (three) -- (four) -- cycle;
\draw[fill=blue,fill opacity = 0.1] (l) -- (five) -- (six) -- cycle;

\draw (h) -- (one) -- (seven) -- cycle;
\draw (h) -- (two) -- (eight) -- cycle;
\draw (m) -- (three) -- (nine) -- cycle;
\draw (m) -- (four) -- (ten) -- cycle;
\draw (l) -- (five) -- (eleven) -- cycle;
\draw (l) -- (six) -- (twelve) -- cycle;

\draw (seven) -- (twelve) -- (thirteen) -- cycle;
\draw (eight) -- (nine) -- (fourteen) -- cycle;
\draw (ten) -- (eleven) -- (fifteen) -- cycle;

\draw (five) -- (eleven) -- (fifteen) -- (sixteen) -- cycle;
\draw (six) -- (twelve) -- (thirteen) -- (seventeen) -- cycle;
\draw (five) -- (six) -- (seventeen) -- (sixteen) -- cycle;
\draw (one) -- (seven) -- (thirteen) -- (eighteen) -- cycle;
\draw[fill=blue,fill opacity = 0.1] (one) -- (two) -- (nineteen) -- (eighteen) -- cycle;
\draw (two) -- (eight) -- (fourteen) -- (nineteen) -- cycle;
\draw (three) -- (nine) -- (fourteen) -- (twenty) -- cycle;
\draw[fill=blue,fill opacity = 0.1] (three) -- (four) -- (twentyone) -- (twenty) -- cycle;
\draw (four) -- (ten) -- (fifteen) -- (twentyone) -- cycle;
\draw[fill=blue,fill opacity = 0.1] (five) -- (six) -- (seventeen) -- (sixteen) -- cycle;

\draw (fourteen) -- (nineteen) -- (twentytwo) -- (twenty) -- cycle;
\draw[fill=blue,fill opacity = 0.1] (eighteen) -- (nineteen) -- (twentytwo) -- cycle;
\draw (thirteen) -- (eighteen) -- (twentytwo) -- (seventeen) -- cycle;
\draw[fill=blue,fill opacity = 0.1] (sixteen) -- (seventeen) -- (twentytwo) -- cycle;
\draw (fifteen) -- (sixteen) -- (twentytwo) -- (twentyone) -- cycle;
\draw[fill=blue,fill opacity = 0.1] (twenty) -- (twentyone) -- (twentytwo) -- cycle;

\draw[dashed, color=blue, opacity = 0.6] (h) -- (twentytwo) -- (m) -- (twentytwo) -- (l);

\end{scope}

\node (h12center) at (5/8, 1/16, 1/16) {};
\node (h12label) at (5/8, 1/8, -1/4) {{\tiny \,\,\, \,\, $(\emptyset,14)$}};
\draw[-{Latex[length=1mm, width=1mm]}, red] (h12label) -- (h12center.center);

\node (121918center) at (5/12, 1/6, 1/6) {};
\node (121918label) at (1/3, 1/4, -1/5) {{\tiny $(1,4)$}};
\draw[-{Latex[length=1mm, width=1mm]}, red] (121918label) -- (121918center.center);

\node (71213center) at (23/64, 1/32, 23/64) {};
\node[rotate=60] (71213label) at (3/8, -1/3, 3/8) {{\tiny $(12,4)$}};
\draw[-{Latex[length=1mm, width=1mm]}, red] (71213label) -- (71213center.center);

\node (h17center) at (5/8, 1/32, 3/32) {};
\node[rotate=60] (h17label) at (5/8, -1/5, 1/8) {{\tiny $(14,2)$}};
\draw[-{Latex[length=1mm, width=1mm]}, red] (h17label) -- (h17center.center);

\node (202122center) at ( 7/32,5/16, 7/32 ) {};
\node[rotate=-60] (202122label) at (-1/8, 5/8, 1/8) {{\tiny  $(\emptyset, 34)$}};
\draw[-{Latex[length=1mm, width=1mm]}, red] (202122label) -- (202122center.center);

\node (15162221center) at (5/16, 13/32, 13/32) {};
\node[rotate=-60] (15162221label) at (-1/4, 3/8, 3/8) {{\tiny \,\, $(\emptyset,234)$}};
\draw[-{Latex[length=1mm, width=1mm]}, red] (15162221label) -- (15162221center.center);

\node (5111516center) at (1/12,1/4, 5/12) {};
\node[rotate=0] (5111516label) at (-1/5, 1/8, 9/16) {{\tiny \,\,\,\,\,\,$(2,34)$}};
\draw[-{Latex[length=1mm, width=1mm]}, red] (5111516label) -- (5111516center.center);

\end{tikzpicture} & \begin{tikzpicture} [scale=\tikzfigscale, thick, tdplot_main_coords]
\coordinate (orig) at (0,0,0);

\coordinate (uniform) at (1/4,1/4,1/4);
\coordinate[label=below:${p_1}$] (h) at (3/4,0,0);
\coordinate[label=below:${p_3}$] (m) at (0,3/4,0);
\coordinate[label=right:${p_2}$] (l) at (0,0,3/4);

\coordinate (one) at (3/4 - 1/8, 0, 1/8);
\coordinate (two) at (3/4 - 1/8, 1/8, 0);
\coordinate (three) at (1/8, 3/4 - 1/8, 0);
\coordinate (four) at (0, 3/4 - 1/8, 1/8);
\coordinate (five) at (0, 1/8, 3/4 - 1/8);
\coordinate (six) at (1/8, 0, 3/4 - 1/8);
\coordinate (seven) at (3/4 - 1/8, 1/40,1/8 - 1/40);
\coordinate (eight) at (3/4 - 1/8, 1/8 - 1/40, 1/40);
\coordinate (nine) at (1/8 - 1/40, 3/4 - 1/8, 1/40);
\coordinate (ten) at (1/40, 3/4 - 1/8, 1/8 - 1/40);
\coordinate (eleven) at (1/40, 1/8 - 1/40, 3/4 - 1/8);
\coordinate (twelve) at (1/8 - 1/40, 1/40, 3/4 - 1/8);
\coordinate (thirteen) at (3/4 - 1/3, 1/3, 0);
\coordinate (fourteen) at (1/3, 3/4 - 1/3, 0);
\coordinate (fifteen) at (0,3/4 - 1/3, 1/3);
\coordinate (sixteen) at (0,1/3, 3/4 - 1/3);
\coordinate (seventeen) at (1/3, 0, 3/4 - 1/3);
\coordinate (eighteen) at (3/4 - 1/3,0, 1/3);
\coordinate (nineteen) at (1/2,1/8, 1/8);
\coordinate (twenty) at (1/3, 1/3, 3/4 - 2/3);
\coordinate (twentyone) at (1/8, 1/2, 1/8);
\coordinate (twentytwo) at (3/4 - 2/3, 1/3, 1/3);
\coordinate (twentythree) at (1/8, 1/8, 1/2);
\coordinate (twentyfour) at (1/3, 3/4 - 2/3, 1/3);
\coordinate (twentyfive) at (1/3, 1/4, 1/2 - 1/3);
\coordinate (twentysix) at (1/4, 1/3, 1/2 - 1/3);
\coordinate (twentyseven) at (1/2 - 1/3, 1/3, 1/4);
\coordinate (twentyeight) at (1/2 - 1/3, 1/4, 1/3);
\coordinate (twentynine) at (1/4, 1/2-1/3, 1/3);
\coordinate (thirty) at (1/3, 1/2 - 1/3, 1/4);

\draw[simplex] (h) -- (m) -- (l) -- (h);

\begin{scope}
\clip (h) -- (m) -- (l) -- (h);

\draw (h) -- (one) -- (seven) -- cycle;
\draw (h) -- (two) -- (eight) -- cycle;
\draw (m) -- (three) -- (nine) -- cycle;
\draw (m) -- (four) -- (ten) -- cycle;
\draw (l) -- (five) -- (eleven) -- cycle;
\draw (l) -- (six) -- (twelve) -- cycle;

\draw[fill=blue,fill opacity = 0.1] (h) -- (seven) -- (nineteen) -- (eight) -- cycle;
\draw[fill=blue,fill opacity = 0.1] (m) -- (nine) -- (twentyone) -- (ten) -- cycle;
\draw[fill=blue,fill opacity = 0.1] (l) -- (eleven) -- (twentythree) -- (twelve) -- cycle;

\draw (one) -- (seven) -- (nineteen) -- (thirty) -- (twentyfour) -- (eighteen) -- cycle;
\draw (two) -- (eight) -- (nineteen) -- (twentyfive) -- (twenty) -- (thirteen); 
\draw (three) -- (nine) -- (twentyone) -- (twentysix) -- (twenty) -- (fourteen);
\draw (four) -- (ten) -- (twentyone) -- (twentyseven) -- (twentytwo) -- (fifteen) -- cycle;
\draw (five) -- (eleven) -- (twentythree) -- (twentyeight) -- (twentytwo) -- (sixteen) -- cycle;
\draw (six) -- (twelve) -- (twentythree) -- (twentynine) -- (twentyfour) -- (seventeen) -- cycle;

\draw[fill=blue,fill opacity = 0.1] (nineteen) -- (twentyfive) -- (uniform) -- (thirty);
\draw[fill=blue,fill opacity = 0.1] (twentyone) -- (twentyseven) -- (uniform) -- (twentysix);
\draw[fill=blue,fill opacity = 0.1] (twentythree) -- (twentyeight) -- (uniform) -- (twentynine) -- cycle;

\draw (twenty) -- (twentysix) -- (uniform) -- (twentyfive) -- cycle;
\draw (twentytwo) -- (twentyseven) -- (uniform) -- (twentyeight) -- cycle;
\draw (twentyfour) -- (twentynine) -- (uniform) -- (thirty) -- cycle;

\draw[dashed, color=blue, opacity = 0.6] (h) -- (uniform) -- (m) -- (uniform) -- (l);

\end{scope}

\node (2101center) at (.6 * .75, 0.1 * .75, 0.3 * 0.75) {};
\node (2101label) at (5/8,-1/4, 3/8) {{\tiny $(3,24,1)$}};
\draw[-{Latex[length=1mm, width=1mm]}, red] (2101label) -- (2101center.center);

\node (2001center) at (.8 * .75, 0.1 * .75, 0.1 * 0.75) {};
\node (2001label) at (7/8,-1/4, 1/8) {{\tiny $(23,4,1)$}};
\draw[-{Latex[length=1mm, width=1mm]}, red] (2001label) -- (2001center.center);

\node (2201center) at (.475 * .75, 0.05 * .75, 0.475 * 0.75) {};
\node (2201label) at (3/8,-1/8, 5/8) {{\tiny $(3,4,12)$}};
\draw[-{Latex[length=1mm, width=1mm]}, red] (2201label) -- (2201center.center);

\node (1011center) at (.405 * .75, 0.405 * 0.75, 0.19 * .75) {};
\node (1011label) at (4/8, 2/8, -1/6) {{\tiny $(2, 134)$}};
\draw[-{Latex[length=1mm, width=1mm]}, red] (1011label) -- (1011center.center);

\node (0011center) at (.25 * .75, 0.5 * 0.75, 0.25 * .75) {};
\node (0011label) at (2/8, 4/8, -1/6) {{\tiny $(12, 34)$}};
\draw[-{Latex[length=1mm, width=1mm]}, red] (0011label) -- (0011center.center);

\end{tikzpicture} & \begin{tikzpicture} [scale=\tikzfigscale, thick, tdplot_main_coords]
\coordinate (orig) at (0,0,0);

\coordinate (uniform) at (1/4,1/4,1/4);
\coordinate[label=below:${p_1}$] (h) at (3/4,0,0);
\coordinate[label=below:${p_3}$] (m) at (0,3/4,0);
\coordinate[label=right:${p_2}$] (l) at (0,0,3/4);

\draw[simplex] (h) -- (m) -- (l) -- (h);

\begin{scope}
\clip (h) -- (m) -- (l) -- (h);

\draw (h) -- (uniform) -- (m) -- cycle;
\draw (h) -- (uniform) -- (l) -- cycle;
\draw (m) -- (uniform) -- (l) -- cycle;

\draw[dashed, color=blue, opacity = 0.6] (h) -- (uniform) -- (m) -- (uniform) -- (l);

\end{scope}

\node (134center) at (.45 * .75, 0.45 * .75, 0.1 * 0.75) {};
\node (134label) at (3/8, 3/8, -1/5) {{\tiny $134$}};
\draw[-{Latex[length=1mm, width=1mm]}, red] (134label) -- (134center.center);

\node (234center) at (.1 * .75, 0.45 * .75, 0.45 * 0.75) {};
\node (234label) at (-1/4, 3/8, 3/8) {{\tiny $234$}};
\draw[-{Latex[length=1mm, width=1mm]}, red] (234label) -- (234center.center);

\node (124center) at (.45 * .75, 0.1 * .75, 0.45 * 0.75) {};
\node (124label) at (3/8, -1/4,3/8) {{\tiny $124$}};
\draw[-{Latex[length=1mm, width=1mm]}, red] (124label) -- (124center.center);

\end{tikzpicture} & \begin{tikzpicture} [scale=\tikzfigscale, thick, tdplot_main_coords]
\coordinate (orig) at (0,0,0);

\coordinate (uniform) at (1/4,1/4,1/4);
\coordinate[label=below:${p_1}$] (h) at (3/4,0,0);
\coordinate[label=below:${p_3}$] (m) at (0,3/4,0);
\coordinate[label=right:${p_2}$] (l) at (0,0,3/4);

\draw[simplex] (h) -- (m) -- (l) -- (h);

\begin{scope}
\clip (h) -- (m) -- (l) -- (h);

\draw (h) -- (uniform) -- (m) -- cycle;
\draw (h) -- (uniform) -- (l) -- cycle;
\draw (m) -- (uniform) -- (l) -- cycle;

\draw[dashed, color=blue, opacity = 0.6] (h) -- (uniform) -- (m) -- (uniform) -- (l);

\end{scope}

\node (134center) at (.45 * .75, 0.45 * .75, 0.1 * 0.75) {};
\node (134label) at (3/8, 3/8, -1/5) {{\tiny $134$}};
\draw[-{Latex[length=1mm, width=1mm]}, red] (134label) -- (134center.center);

\node (234center) at (.1 * .75, 0.45 * .75, 0.45 * 0.75) {};
\node (234label) at (-1/4, 3/8, 3/8) {{\tiny $234$}};
\draw[-{Latex[length=1mm, width=1mm]}, red] (234label) -- (234center.center);

\node (124center) at (.45 * .75, 0.1 * .75, 0.45 * 0.75) {};
\node (124label) at (3/8, -1/4,3/8) {{\tiny $124$}};
\draw[-{Latex[length=1mm, width=1mm]}, red] (124label) -- (124center.center);

\end{tikzpicture}\\ 
\end{tabular}\caption{Visualizations of the minimizers of the losses (embedded by) $\Li{2}$, $\Li{3}$, $\Li{4}$, and $\Lk$ with $n=4$ and $k \in \{2,3\}$, fixing $p_4 = 1/4$. The dashed blue lines give sets of distributions $p$ corresponding to the same report $u$ such that $u \in \prop{\Lk}(p)$. As we must link reports deterministically, we want each of the bold, black cells to be fully contained in a cell from the blue dashed cells. Blue regions cross the dashed blue lines and suggest where deciding how to construct a link $\psi$ is ambiguous, as $|\topkset(u)| > 1$.  White regions are therefore where the surrogate and any top-$k$ link are consistent, e.g., $\P^{(i)}$. On the right, $\Lk$ shows our proposed surrogate that is consistent for top-$k$ classification, demonstrated by no blue regions.  Each of the cells corresponds to a subset of distributions where exactly $k$ reports are optimal. }
\label{tab:loss-slices}
\end{table*}

\subsection{Analysis of $\Li{2}$}\label{sec:L2}
The surrogate $\Li{2}$ proposed by~\citet{lapin2016loss} is given by
\begin{align}
\restatableeq{\Ltwo}{\Li{2}(u,y)~=~\left(1 - u_y + \frac{1}{k} \sum_{i=1}^k (u - e_y)_{[i]} \right)_+~.~}{eq:psi-2}
\end{align}
We will derive a discrete loss $\elli 2$ in eq.~\eqref{eq:hat-ell-2} that $\Li{2}$ embeds, and then use it to characterize the set of distributions $\P^{(2)}$ on which $(\Li{2}, \link)$ is consistent with respect to $\lk$.
See \S~\ref{appendix:L2} for all omitted details.

By our strategy outlined above, we begin with the set $U$ (eq.~\eqref{eq:U-inf-rep}), which is representative for $\Li 2$.
We then construct the bounded region $\Upostwo \subset U$ in which the positive part operator in eq.~\eqref{eq:psi-2} is not activated, and show $\Upostwo$ is representative.
We next partition $\Upostwo$ into polytope regions over which $\Li{2}$ is affine.
When restricting to $\Upostwo$, the only way $\Li{2}(\cdot,y)$ fails to be affine is in the top-$k$ elements of a prediction $(u - e_y)$ changing.
Observe that, up to tie-breaking, the top $k$ elements of $(u-e_y)$ are the same as the top $k$ elements of $u$ if and only if $u_y \geq 1 = 1 + u_{[k+1]}$.
$\Li 2$ is therefore affine on regions where $\sign(u_i - 1)$ is constant for all $i \in \{1, \ldots, k\}$. 
Further examining these affine regions reveals that their vertices are the points $u \in \reals^n_+$ such that $u_i \in \{0,1, c(u)\}$ for a particular value $c(u) > 1$ that depends on how many entries of $u$ are nonzero and how many are strictly greater than 1.

Taking the union of these vertices, we arrive at a finite representative set for $\Li 2$.
Theorem~\ref{thm:polyhedral-embeds-discrete} now states that $\Li 2$ embeds $\Li 2$ restricted to this vertex set.
To state this discrete loss more intuitively, we simply reparameterize these vertices, letting $M$ be the set of entries equal to 1, and $H$ the set strictly greater than 1.
Letting $\Ri 2$ be the set of valid pairs $(H,M)$, namely disjoint and with $|H\cup M| \leq k$, we arrive at the following discrete loss $\elli 2:\Ri 2\times \Y\to\reals$ embedded by $\Li 2$.
\begin{align} \label{eq:hat-ell-2}
    \elli 2((H,M), y) = \begin{cases}
    0 & y \in H \\
     \frac{|H| + |M| - 1}{k-|H|}
     & y \in M\\
     \frac{|H| + |M| - 1}{k-|H|} + \frac{k+1}{k}
     & \text{otherwise}
    \end{cases}\hspace*{-10pt}
\end{align}
One can regard $H$ as the ``high labels'', with high likelihood of being the ground truth label, and $M$ the ``medium labels'', with some likelihood.
One therefore attains loss $0$ if they were highly confident in the ground truth label, and accumulate a loss that grows in the size of $H$ and $M$ otherwise.

By our observations above, consistency with respect to top-$k$ is achieved whenever the optimal report is some $(H,M)$ with $|H\cup M| = k$. 
This condition can be written as follows, where $h^*(p) = \max \{i \in \{0,\ldots, k\} \mid  p_\i > \frac{1 - \sigma_{i-1}(p)}{k-(i-1)} \}$.

\begin{restatable}{corollary}{psitwoconsistent}\label{cor:psi-2-inconsistent-region}
    Define
    \begin{align}
\P^{(2)} := \left\{p \in \simplex \mid p_\k > \frac{(1-\sigma_{h^*(p)}(p))}{(k+1)(k-h^*(p))} \right\}~.~\label{eq:consistent-set-L2}
\end{align}
    $\Li{2}$ is consistent with respect to $\lk$ on $\P^{(2)}$.
\end{restatable}

\subsection{Analysis of $\Li{3}$}\label{sec:L3}
\citet{lapin2016loss} give two convex upper bounds on the proposed top-$k$ surrogate  from~\citep{lapin2015top}: $\Li{2}$ studied in \S~\ref{sec:L2}, and $\Li{3}$, defined as follows.
\begin{align}
\restatableeq{\Lthree}{\Li{3}(u,y) = \frac 1 k \sum_{i=1}^k \left[ 1 - u_y + (u-e_y)_{[i]} \right]_+}{eq:L3}
\end{align}
While similar to $\Li{2}$, the placement of the positive part operator changes the analysis of the surrogate significantly. 
See \S~\ref{appendix:L3} for all omitted details.

As above, it suffices to identify sources of non-affineness on $U$ (eq.~\eqref{eq:U-inf-rep}) to construct a finite representative set for $\Li{3}$.
Non-affineness of $\Li{3}$ is introduced by the positive part operator and the ordering of the top-$k$ elements of a prediction $u \in U$.
Unlike $\Li 2$, the positive part operator is applied to each term of the summand, so we cannot immediately ignore this operator by restricting to a bounded representative region.
Instead, let us simultaneously fix (1) a set $S \in \Rk$ to be indices of the top-$k$ elements of $u$, and (2) sets $\vec V = \{V_y \subseteq S \setminus \{y\} \mid y\in\Y\}$ corresponding to induces when the positive part operator is not activated for $\Li 3(u,y)$.
For any such $S, \vec V$, therefore, we define the region $A^{S, \vec V}$ to be all points $u\in U$ with (1) $S \in \topkset(u)$ and (2) for all $y\in\Y$, we have $u_i + 1 \geq u_y$ for all $i\in V_y$, and $u_i + 1 \leq u_y$ for all $i\notin V_y$.
By the above reasoning,
$\Li{3}$ is affine on the set $A^{S, \vec V}$ for each choice of $S, \vec V$.

The union of the vertices of each $A^{S, \vec V}$ region is therefore a finite representative set, and $\Li 3$ embeds $\Li 3$ restricted to these vertices.
Upon inspection of the geometry of the $A^{S, \vec V}$ regions, we show that the vertices of each are in fact a subset of $\mathbb{Z}_k^n$.
A more intuitive form for this discrete loss can therefore be expressed in terms of ordered partitions, where index $i$ is in the $j^\text{th}$ partition $Q_j$ when $u_i = j$.
Formally, we reparameterize the vertices as ordered partitions $Q \in \Ri{3}$, where
\begin{multline*}
    \Ri{3} = \{ Q\!=\!(Q_0, \ldots, Q_s) \mid s\!\leq\! k, Q_i \cap Q_j = \emptyset \, \forall i \neq j, \\ |Q_1, \cup \ldots \cup Q_s| \leq k, Q_i \neq \emptyset \, \forall i\}~.
\end{multline*}

We now have that $\Li{3}$ embeds $\elli 3:\Ri 3\times \Y\to\reals$, given by
\begin{align*}
\elli 3(Q, y) &=\!
\begin{cases}
\frac 1 k \!\left(|Q_j| - 1 + \sum\limits_{i > j} |Q_i| (i-j +1) \right) \!\!\!& j\!>\!0\\
\frac 1 k \sum\limits_{i=1}^s |Q_i|(i+1)  & j\!=\!0
\end{cases}
\end{align*}
where $y \in Q_j$.
For intuition, $\elli 3$ allows for predictions with more granularity than $\elli 2$, where the higher index $i$ of the partition $Q_i$ is, the more confident one is in outcomes in $Q_i$.
The punishment for error again grows in the number of indices one reports high confidence in, as well as the number of partitions. 

In order to characterize the regions where $(\Li{3},\link)$ is consistent with respect to $\lk$, we can study where $\elli 3$ can be unambiguously linked to $\lk$.
In particular, one can do so for any $p \in \simplex$ such that $|Q_0| = n-k$ for $Q \in \prop{\elli 3}(p)$.

\begin{restatable}{corollary}{ellthreeconsistent}
$\Li{3}$ is consistent with respect to $\lk$ on $\P^{(3)}  = \{p \in \simplex \mid p_{[k+1]} > \frac 1 {k+1} \wedge \frac{\sum_{i = k+1}^n p_\i}{k-1} \geq p_\k\}$.
\end{restatable}

\subsection{Analysis of $\Li{4}$}
\label{sec:L4}
Observing that $\Li{2}$ and $\Li{3}$ are inconsistent with respect to $\lk$, \citet{yang2018consistency} propose $\Li{4}$ as in eq~\eqref{eq:L4}, changing the summation from elements of $(u-e_y)$ to elements of $u_{\smy} \in \reals^{n-1}$: the elements of $u$ excluding $u_y$.
See \S~\ref{appendix:L4} for all omitted details.
\begin{align}\label{eq:L4}
\Li{4}(u,y) &=\left(1-u_y+\frac{1}{k}\sum_{i=1}^k(u_{\backslash y})_{[i]}\right)_+
\end{align}

Again following the strategy outlined above, we begin with the set $U$, which is representative for $\Li{4}$.
Here we also further restrict to the set of points $U_+^{(4)} \subseteq U$ yielding a nonnegative argument to the positive part operator, and show that $U_+^{(4)}$ is also representative for $\Li{4}$.
Within $U_+^{(4)}$, we observe that the only way $\Li{4}(\cdot,y)$ fails to be affine is when the top $k$ elements of $u_{\smy}$ change.
Since all elements of $U$ have at most $k$ nonzero entries already, it therefore suffices to select a subset $T$ of nonzero indices.
For any $T \subseteq [n]$ with $|T|\leq k$, let us therefore define the set $A^T$ to be all points $u\in \reals^n$ such that $0\leq u_i\leq 1+\frac{1}{k}\sum_{j\in T,j\neq i}u_j$ for $i\in T$, and $u_i = 0$ for $i \notin T$.
For any $p \in \simplex$, the function $ u \mapsto \inprod{\Li{4}(u, \cdot)}{p}$ is affine on each region $A^T$, and moreover, they partition the representative set $U_+^{(4)}$.

Taking the union of vertices of each $A^T$ set, we arrive at a finite representative set for $\Li{4}$.
Carefully examining the geometry of the $A^T$ sets, one sees that these vertices are the points $u \in \reals^n$ such that each element is either $0$ or $\frac{k}{k+1-|T|}$.
Therefore, the finite representative set for $\Li{4}$ can be reparameterized as $\Ri{4} = \left\{T \subseteq [n] \mid |T| \leq k \right\}$,
and thus $\Li{4}$ embeds $\elli 4:\Ri{4}\times\Y\to\reals$ given by
\begin{align*}
\elli 4(T,y)=
\begin{cases}
 0  & y\in T\\
 \frac{k+1}{k+1-|T|} & y\notin T
\end{cases}~.~
\end{align*}
Intuitively, $\elli{4}$ is a variant of top-$k$ where one may report any set of labels of size $m \leq k$, and the stakes for being incorrect increase in $m$.
Therefore, the loss incentivizes one to report smaller sets only when sufficiently confident.

Following this intuition, consistency therefore arises whenever the conditional label distribution does not lead to such high confidence that the optimal report is a set of size $m < k$.
We characterize such distributions as follows.
\begin{restatable}{corollary}{ellfourinconsistent}
$\Li{4}$ is consistent with respect to $\lk$ on $\P^{(4)} := \{p \in \simplex \mid p_\k > 1-\sigma_k(p) \}$.
\end{restatable}

\section{A New Consistent Surrogate}\label{sec:embedding-construction}

\citet{yang2018consistency} show that the polyhedral surrogates analyzed in \S~\ref{sec:previous-surrogates} are not consistent for top-$k$.
They further suggest that perhaps \emph{no} polyhedral surrogate can be consistent.
On the other hand, the embedding framework of \citet{finocchiaro2019embedding,finocchiaro2022embedding} shows that every discrete loss has a consistent polyhedral surrogate.
As their result is constructive, we apply it to the top-$k$ loss $\lk$, giving the first consistent polyhedral surrogate,  $\Lk$, for the problem (\S~\ref{subsec:new-surrogate-def}).
The embedding framework relies on constructing a link from scratch, rather than using a pre-specified link function.
As such, in principle their surrogate construction could yield a surrogate which is not consistent when paired with $\psi_k$, but only with a different link entirely.
Interestingly, we further show that in particular $(\Lk,\psi_k)$ is consistent with respect to $\lk$ (\S~\ref{subsec:canonical-calibrated}).

\subsection{Formulating $\Lk$} \label{subsec:new-surrogate-def}

To show that every discrete loss is embedded by a consistent polyhedral surrogate, \citeauthor{finocchiaro2022embedding} give the following construction.
Their construction echoes similar constructions in the literature (cf.\ \citet{asif2015adversarial}, \citet{farnia2016minimax}, \citet{fathony2016adversarial},  \citet{duchi2018multiclass}.)
Recall that the \emph{Bayes risk} of a loss $\ell:\R\times\Y\to\reals$ is the function $\risk{\ell}:\simplex\to\reals$, $\risk{\ell}:p\mapsto\min_{r\in\R} \inprod{p}{\ell(r,\cdot)}$.
\begin{theorem}[{\citet[Theorem 4]{finocchiaro2022embedding}}]
    \label{thm:surrogate-construction}
    Any discrete loss $\ell: \R \times \Y \to \reals_+$
    is embedded by the consistent surrogate $L(u, y) = (-\risk{\ell})^*(u) - u_y $
    where
    $(\cdot)^*$ denotes the convex conjugate.
\end{theorem}

The Bayes risk of $\lk$ is
\[ \underline{\lk}(p) = \inf_{S \in \Rk}\inprod{p}{\lk(S,\cdot)} = 1 - \sumk(p) ~. \]

By Theorem \ref{thm:surrogate-construction}, the following loss function $\Lk$ therefore embeds $\lk$, with consistency (for some link function) following from Theorem \ref{thm:polyhedral-embeds-discrete}.
\begin{align}
	\Lk(u, y) 
	&= (-\risk{\lk})^*(u) - u_y \nonumber \\
	&= \sup_{p \in \simplex}\left(\inprod{p}{u} + \underline{\lk}(p)\right) - u_y  \nonumber \\
	&= \sup_{p \in \simplex}\left(\inprod{p}{u} + 1 - \sumk(p)\right) - u_y ~.     \label{eq:new-surrogate} \\
\intertext{Choosing $p$ to be uniform on the $m$ largest indices of $u$ (which we justify in \S~\ref{subsec:new-surrogate-max-form}), this expression simplifies to}
    &= \max_{1\leq m \leq n} \left\{ \tfrac {\sigma_m(u)}{m} +  \left(1-\tfrac k m\right)_+ \right\} - u_y~. \label{eq:neq-surrogate-max-form}
\intertext{Since $\frac{\sigma_m(u)}{m}$ is non-increasing in $m$, and $1 - \frac{k}{m} \leq 0$ for $0 < m \leq k$, the $m = 1$ case will dominate the $1 < m \leq k$ cases. Therefore, we can further simplify the loss,}
    &= \max\left\{u_{[1]}, \max_{k < m \leq n} \left\{ \tfrac {\sigma_m(u)}{m} +  1 - \tfrac{k}{m} \right\} \! \right\}\! - u_y~. \nonumber
\end{align}
In this form, it is clear to see that the surrogate is piecewise linear, as a maximum of affine functions (recall that $\sigma_m$ can itself be written as a maximum).

\subsection{The Argmax Link is Calibrated} \label{subsec:canonical-calibrated}

From Theorem \ref{thm:polyhedral-embeds-discrete}, there exists some link function $\psi: \reals^n \rightarrow \Rk$ mapping the report space of $\Lk$ back to the that of $\lk$, such that $(\Lk,\psi)$ is consistent with respect to $\lk$.
It remains to actually find this link $\psi$.
In fact, we will show that one can take $\psi = \link$, the canonical argmax link.

Recall that consistency is characterized by calibration (Definition \ref{def:calibration}), which says that linking to a non-$\lk$-optimal report should be strictly $\Lk$-suboptimal.
To show that $\link$ is calibrated, we in turn use another equivalent condition, that $\link$ be $\epsilon$-separated \citep[Definition 8]{finocchiaro2022embedding} with respect to $\lk$ and $\Lk$.
Recall that all minimizable losses elicit a property (Definition \ref{def:property}), which is just a map from distributions to all optimal reports under that loss.

\begin{definition}
  \label{def:epsilon-separation}
  Given a discrete loss $\ell: \R \times \Y \to \reals_+$ and surrogate $L:\reals^d \times \Y \to \reals_+$, let $\Gamma = \prop{L}$ and $\gamma = \prop{\ell}$ be their respective properties. 
  The link $\psi: \reals^d \rightarrow \R$ is $\epsilon$-separated with respect to $(L, \ell)$ if for all $p \in \simplex$, $u \in \Gamma(p)$, and $u'\in\reals^d$ such that $\psi(u') \not \in \gamma(p)$, we have $\|u - u'\|_\infty \geq \epsilon$.
\end{definition}
Calibration and $\epsilon$-separation are equivalent for polyhedral surrogates \citep[Theorem 5]{finocchiaro2022embedding}.

To show $\epsilon$-separation, we first must characterize the properties of $\lk$ and $\Lk$.
Eq.~\eqref{eqn:topk-property} gives us $\prop{\lk} = \gamma_k$.
Let $\propk = \prop{\Lk}$.

Recall that the report space of $\lk$ is $\Rk = \{ S \subseteq \Y \mid |S| = k \}$. 
Let $\mathcal{T} = \{ \ones_S \mid S \in \Rk\}$ be the set of indicators for the elements of $\Rk$.
Then, $\topkvec(u) = \argmax_{t \in \mathcal{T}} \inprod{t}{u}$ is the set of possible indicators of the top $k$ elements of $u$.
Note that $|\topkvec(u)| > 1$ if and only if $u_{[k]} = u_{[k+1]}$.

\begin{lemma} \label{lemma:Gamma}
	Let $\cone$ denote the convex cone. Then,
	\[ \propk(p) = \hull(\topkvec(p)) - \cone\{\ones_{i} \mid p_i = 0 \} + \bigcup_{\alpha \in \reals} \{\alpha \ones\}~.\]
	
\end{lemma}
The proof, deferred to \S~\ref{appendix:Gamma-lemma-proof}, relies on the connection between $\propk$ and the subgradients of $-\risk{\lk}$.
With this characterization of $\gamma_k$ and $\propk$, we can prove that $\link$ is calibrated.
\begin{theorem} \label{thm:topk-link-calibrated}
    $(\Lk,\link)$ is calibrated with respect to $\lk$.
\end{theorem}
\begin{proof}
First, we show $\link$ is $\epsilon$-separated with respect to $\propk$ and $\gamma_k$.
Let $\epsilon = \frac{1}{2n}$. 
Fix any $p \in \simplex$, and choose any $u \in \propk(p)$. 
Choose $\alpha$ such that $u - \alpha \ones \in \hull(\topkvec(p)) - \cone\{\ones_{i} \mid p_i = 0 \}$.
We need to show for every $u'$ with $\link(u') \not \in \gamma_k(p)$, $\|u - u'\|_\infty \geq \frac{1}{2n}$. 

\emph{Case 1:  $p_\k > 0$.}
Since $u \in \propk(p)$, Lemma \ref{lemma:Gamma} implies every element of $u$ is at most $1 + \alpha$, so we have $\sigma_{k-1}(u) \leq (k-1)(1 + \alpha)$.
Let $S = \support(\gamma_k(p))$, the set of indices $i$ with $p_\i \geq p_\k > 0$.
Lemma \ref{lemma:Gamma} also implies $\sum_{i\in S} u_i = k + \alpha|S|$.
Since $u_\k$ is the largest element of $S$ that is not in the top $k-1$ elements of $u$, we have 
\begin{align*}
    u_\k 
    &\geq \frac{\left(\sum_{i\in S} u_i\right) - \sigma_{k-1}(u)}{|S| - (k-1)}\\
    &= \frac{k + \alpha |S| - (k-1)(1 + \alpha)}{|S| - (k-1)}\\
    &= \frac{1}{|S| - (k-1)} + \alpha\\
    &> \frac{1}{n} + \alpha ~.    
\end{align*}

Now, pick any $u'$ such that $\link(u') \not \in \gamma_k(p)$. 
Since $\link(u')$ is some top-$k$ index set of $u'$, and by eq. \eqref{eqn:topk-property} $\gamma_k(p)$ is every possible top-$k$ index set of $p$, then there must be some index $j \in \link(u')$ such that $p_j < p_\k$. 
Then by Lemma \ref{lemma:Gamma}, $u_j \leq \alpha$.

We proceed by contradiction.
Assume $\|u - u'\|_\infty < \frac{1}{2n}$. 
Therefore for every index $i$, we have $|u_i - u'_i| < \frac{1}{2n}$.
Since $u_\k > \alpha + \frac{1}{n}$, for every $i \in \link(u)$, we must have $u'_i > u_i - \frac 1 {2n} \geq u_\k - \frac 1 {2n} \geq \alpha + \frac{1}{2n}$. 
Since $u_j \leq \alpha$, we also must have $u'_j < \alpha + \frac{1}{2n}$.
However, that means there are $|\link(u)| = k$ elements of $u'$ which are larger than $u'_j$, so $j \not \in \link(u')$, a contradiction.
Therefore, $\|u - u'\|_\infty \geq \frac{1}{2n}$.

\emph{Case 2: $p_\k = 0$.}
Let $S = \{i | p_\i > 0\}$. 
Therefore, for all $i\in S$, $u_i = 1 + \alpha$. 
Since $p_\k = 0$, $S$ must be contained by element of $\gamma_k(p)$.
Choose any $u'$ such that $\link(u') \not \in \gamma_k(p)$. 
By eq \eqref{eqn:topk-property} every element of $\gamma_k(p)$ contains $S$, so there must be some index $j \in S$ such that $u'_j \leq u'_\k$.

We again proceed by contradiction, and assume $\|u - u'\|_\infty < \frac{1}{2n}$.
Since $u_j = 1 + \alpha$, we must have $u'_j > 1 + \alpha - \frac{1}{2n}$.
However, since $u'_j \leq u'_\k$, there must be $k - (|S| - 1)$ elements of $u'$ that are greater than $u'_j$ but not in $S$.
Formally, choose any set $T \subseteq ([n] \setminus S) \cup \link(u')$ with $|T| = k - (|S| - 1)$.
For every $i \in T$ we have $u'_i > u'_j$, so 
\begin{align*}
    \sum_{i \in T} u_i 
    &\geq \left(1 + \alpha - \frac{1}{2n}\right)|T| \\
    &= \left(1 + \alpha - \frac{1}{2n}\right)(k - |S| + 1) \\
    &= \left(1 + \alpha\right)(k - |S|) + \alpha + 1 - \frac{k - |S| + 1}{2n} \\
    &> (k - |S|)(1 + \alpha) + \alpha ~.
\end{align*} 
However, by Lemma \ref{lemma:Gamma}, the maximum sum of any $k - |S| + 1$ elements of $[n] \setminus S$ is $(k - |S|) + (k - |S| + 1) \alpha = (k - |S|)(1 + \alpha) + \alpha$, a contradiction. Thus, $\|u - u'\|_\infty \geq \frac{1}{2n}$.

Therefore, in either case, $\link$ is $\epsilon$-separated with respect to $(\propk, \gamma_k)$.
Finally, by \citet[Theorem 5]{finocchiaro2022embedding}, $(\Lk, \link)$ is calibrated with respect to $\lk$.
\end{proof}

\section{Numerical Comparison}\label{sec:regret-comparison}

\begin{figure}[t]
\begin{center}
  \includegraphics[clip,width=\columnwidth]{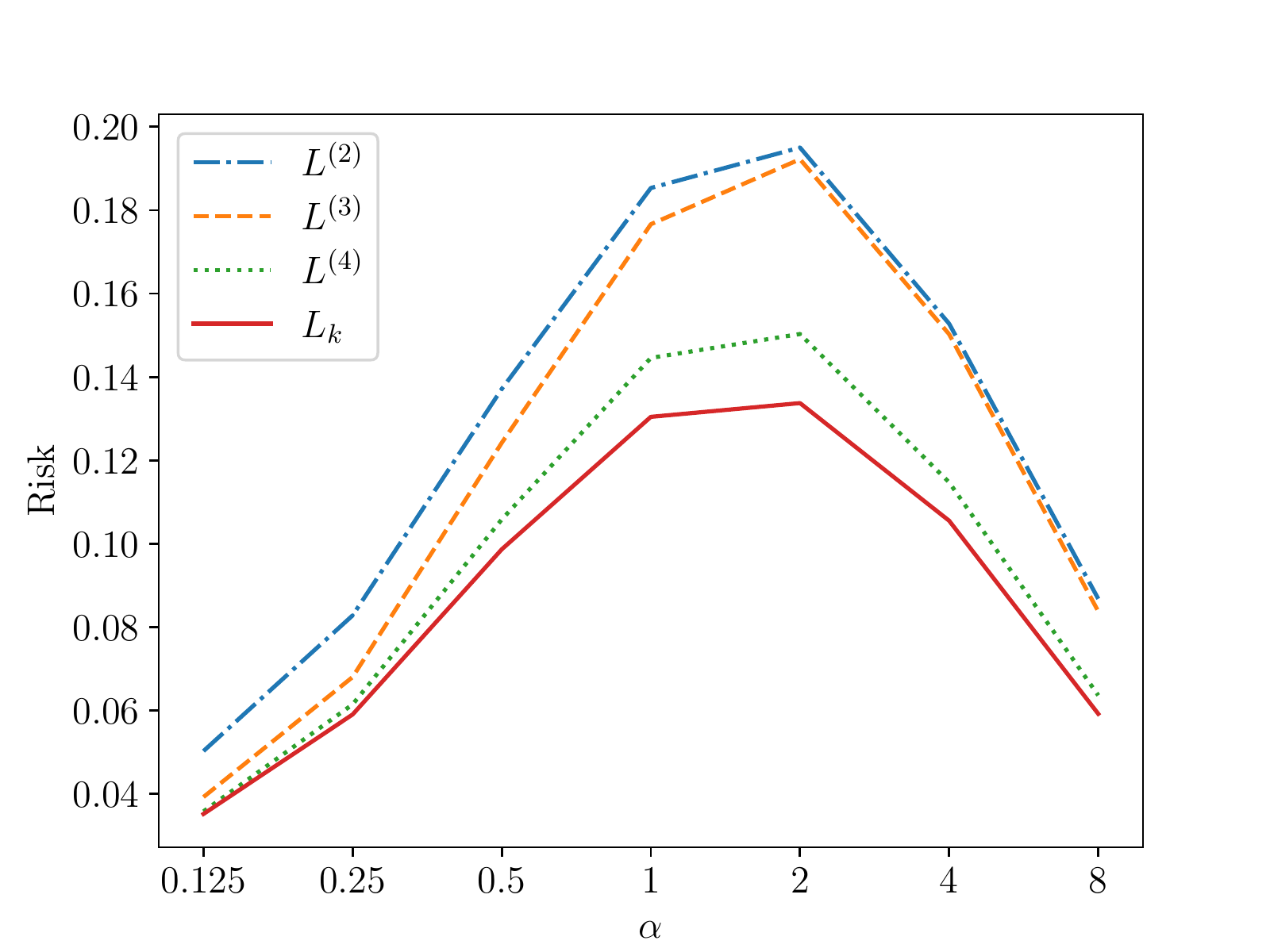}%
  \\

  \includegraphics[clip,width=\columnwidth]{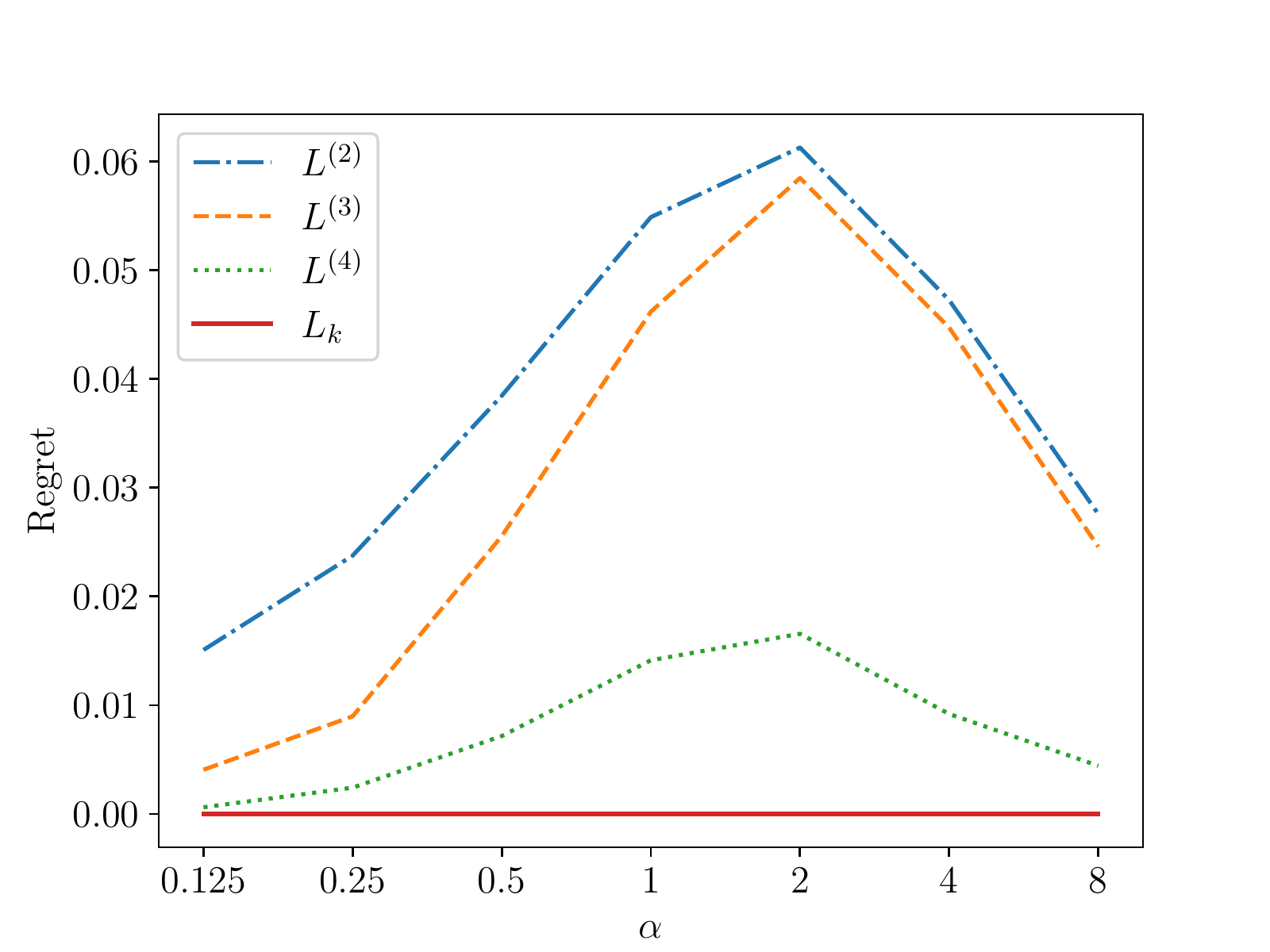}%
\end{center}
\caption{
    The top-$k$ risk (top) and regret (bottom) from surrogate risk minimization of $\Li 2$, $\Li 3$, $\Li 4$, and $\Lk$, for $n=5$ and $k=3$. 
    For each choice of $\alpha$, 1000 conditional label distributions were drawn from $\text{Dirichlet}(\alpha, \alpha, 1, 1, 1)$. 
}
\label{fig:regret}
\end{figure}

We have seen that $\Lk$ is consistent for top-$k$ classification, while $\Li{2}, \Li{3}$, and $\Li{4}$ are not.
In general, therefore, we expect these inconsistent losses to have worse top-$k$ performance than $\Lk$.
We now quantify this gap for the case $n=5$ and $k=3$, by computing the expected difference in top-$k$ loss obtained as a result of optimizing each of the four surrogates.

Recall from Definition~\ref{def:property} that we have $\prop{L}(p) = \argmin_{u \in \reals^n} \inprod{p}{L(u, \cdot)}$ as the minimizers of the expected loss of $L$ under $p$.
For each surrogate $L$ we measure their expected \emph{risk}: the top-$k$ loss obtained by optimizing $L$
\[\Risk(L) = \underset{p\sim D}{\E}\left[\inprod{p}{\lk\left(\link\left(\prop{L}(p)\right),\cdot\right)}\right] ~,\]
and \emph{regret}: the risk minus the true optimal top-$k$ loss.
\[\Regret(L) = \Risk(L) - \underset{p\sim D}{\E}\left[ \argmin_{r \in \R_k} \inprod{p}{\lk(r, \cdot)} \right] ~.\]

Here $p$ is a conditional label distribution, which we draw from $D = \text{Dirichlet}(\alpha, \alpha, 1, 1, 1)$, with $\alpha$ varied from $2^{-3}$ to $2^3$.
We take the $\link$ that breaks ties lexicographically.
The results of these trials are shown in Figure~\ref{fig:regret}.

When $\alpha$ is large, $D$ concentrates on conditional label distributions with most of their weight on the first two labels, and for small $\alpha$, it concentrates on those with weight on the last three.
As $k=3$, we expect all surrogates to perform well in these regimes, since it is relatively easy to select the most likely labels. 
For intermediate values, the distribution is closer to uniform, and the loss increases for all surrogates.
However, the inconsistent surrogates incur the largest increase, and therefore largest regret, as they are more likely to link to a suboptimal set when $p_\k$ is close to $p_\kp$.

As expected, $\Lk$ incurs no regret, since it is consistent.
We also see that of the inconsistent surrogates, $\Li{2}$ incurs the most regret, while $\Li{4}$ incurs the least.
This observation aligns with Table \ref{tab:loss-slices}, which shows that $\Li{2}$ has the largest inconsistent regions, while $\Li{4}$ has the smallest.

Next, we verify this performance empirically. 
We fix $p = (.15, .15, .15, .2, .35)$, a point where $\Li{2}, \Li{3}$, and $\Li{4}$ are inconsistent.
For each value of $\alpha$, we sample 10000 conditional label distributions $p_i \sim \text{Dirichlet}(\alpha p)$; we take the feature vector $x_i=p_i$ and draw the label $y_i \sim p_i$.
For each dataset and each surrogate loss function, we train a linear model for 200 epochs using Adam with a learning rate of 0.01.
Finally, for each $\alpha$, we create a test set with 1000 samples in the same fashion.
We then compute the top-$k$ loss of the model trained for each surrogate loss, and plot the results in Figure \ref{fig:experiment}.

For large $\alpha$, the conditional labels are concentrated on a region where $\Lk$ is consistent but the other surrogate losses are not.
In this regime, $\Lk$ clearly obtains a better top-$k$ test loss. 
For smaller $\alpha$, the conditional distributions are more evenly distributed on $\simplex$, and in this regime $\Lk$ actually performs worse than the inconsistent surrogates.
One explanation for this worse performance could be the shallowness of its gradients.

\begin{figure}[t]
\begin{center}
\includegraphics[clip,width=\columnwidth]{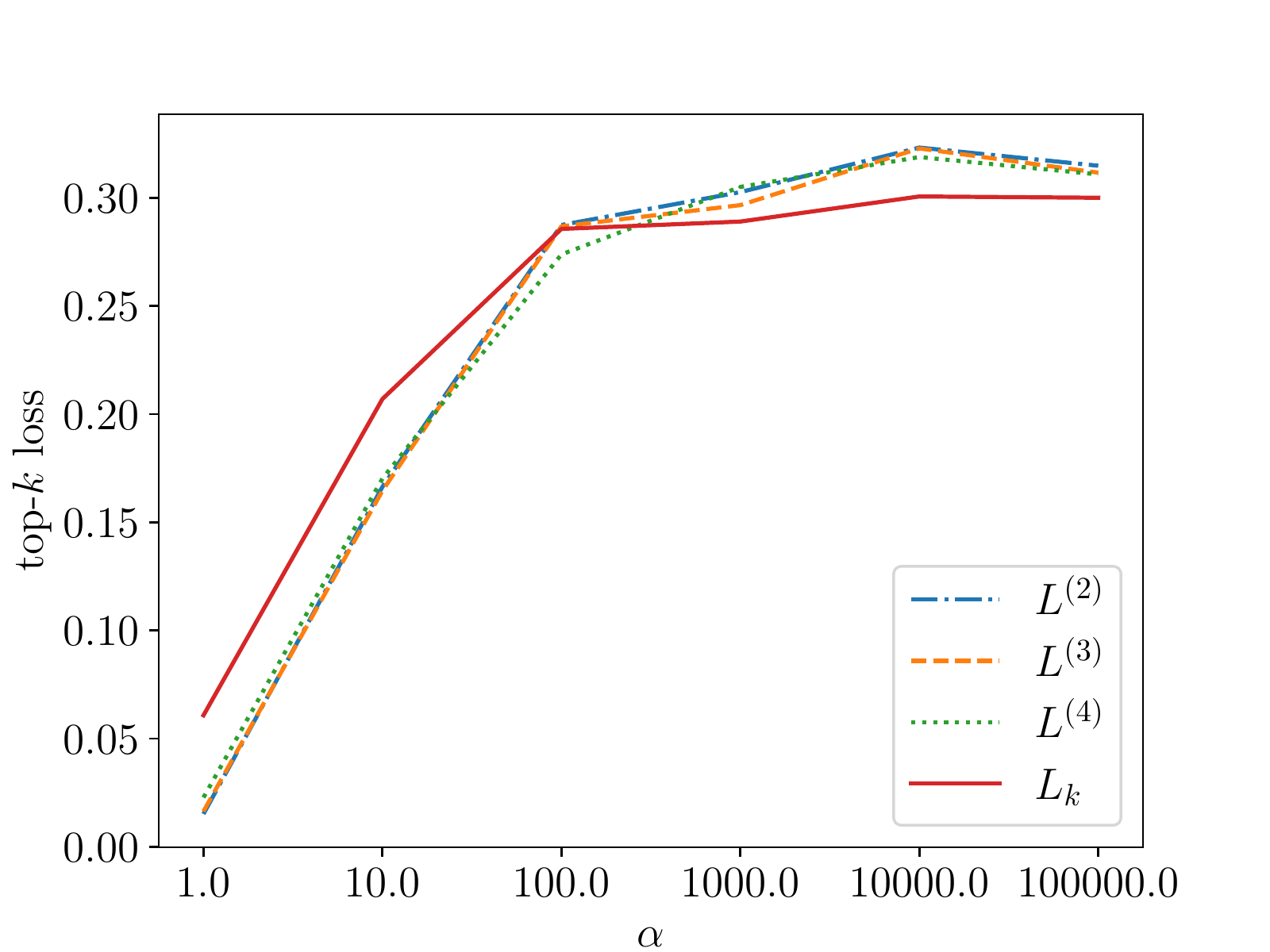}%
\end{center}
\caption{
    The empirical top-$k$ test loss for each loss trained on a dataset with conditional label distributions sampled from $\text{Dirichlet}(\alpha p)$.
}
\label{fig:experiment}
\end{figure}

\section{Discussion}\label{sec:conclusion}

In \S~\ref{sec:previous-surrogates}, we apply the embedding framework of~\citet{finocchiaro2019embedding,finocchiaro2022embedding} to analyze previously proposed, yet inconsistent, surrogates for top-$k$ classification.
The goal of this analysis is two-fold: first, to uncover the discrete losses for which these surrogates are consistent, and second, to characterize distributional conditions sufficent to render them consistent for top-$k$ classification.
We believe this general line of inquiry will be useful for other polyhedral surrogates in the literature known to be inconsistent for their desired target.
In particular, while it is clearly useful to understand the circumstances in which these surrogates would be consistent, we also believe it would be useful to uncover the variants of the intended target which are embedded by these inconsistent surrogates.

To illustrate, consider the surrogate $\Li 4$, analyzed in \S~\ref{sec:L4}.
We showed $\Li 4$ to be consistent for the target loss $\elli 4(T,y) = \frac{k+1}{k+1-|T|}\ones\{y\notin T\}$, which allows one to predict any set of labels $T$ with $|T|\leq k$.
While $\Li 4$ is therefore consistent for top-$k$ only when optimal sets $T$ have size $k$, in practice, the extra flexibility to report smaller sets may be of use.
That is, while common practice is to use $\Li 4$ with the argmax link $\link$, which always yields a set of size $k$, it may be advantageous to use a link $\psi^{(4)}$ that makes $\Li 4$ consistent for $\elli 4$, which could link to sets strictly smaller than $k$.
For example, suppose a search engine has $k = 10$ spaces to show on the first page, but given a specific query $x$, the model $h(x)$ links to $T = \psi^{(4)}(h(x))$ where $|T| = 7$.
Given this information, the search engine may prefer to show only the results in $T$ to reduce visual clutter, or perhaps serve advertisements in the remaining 3 slots.
It is of course rare that a practical decision problem lines up exactly with the canonical discrete loss studied by machine learning researchers---exploring the variants of these canonical problems lurking behind inconsistent polyhedral surrogates may therefore be a useful line of research.
We expect the general technique outlined in \S~\ref{sec:previous-surrogates} would apply readily to other such surrogates.

In \S~\ref{sec:embedding-construction}, we gave the first polyhedral surrogate that is consistent for top-$k$ classification.
This result contributes to an ongoing discussion in the literature about the relative benefits of smooth and polyhedral surrogates.
While it has been suggested that no polyhedral surrogate could be consistent for top-$k$, our surrogate emphasizes the broader finding of \citet{finocchiaro2022embedding}, that in fact \emph{every} discrete target loss has a consistent polyhedral surrogate.
Moreover, any smooth proper loss, with an appropriate link, suffices as a smooth surrogate~\cite{williamson2016composite}.
The question is therefore not one of existence but of when and why smooth surrogates or polyhedral surrogates may be preferable.
In particular, an important open direction is to study the relationship between smoothness, consistency, convergence rates, and excess risk tradeoffs for top-$k$ classification, as well as other discrete prediction tasks.

Finally, while we give the first polyhedral surrogate that is consistent for top-$k$, it remains to compare it to other surrogates in practice beyond our limited experiments.

\section*{Acknowledgements}
The authors would like to thank Enrique Nueve and the anonymous reviewers for their helpful suggestions. We also thank Forest Yang and Sanmi Koyejo for providing implementations of previously studied surrogates.
This material is based upon work supported by the National Science Foundation under Grant No.\ IIS-2045347.

\newpage
\bibliographystyle{plainnat}
\bibliography{extra,topk}

\newpage
\appendix
\onecolumn

\section{Additional Derivations for $\Li{2}$} \label{appendix:L2}

Throughout this section, consider the surrogate loss 
\[ \Ltwo \]

We proceed as follows: find a bounded representative region for $\Li 2$, find the subsets of that region on which $u \mapsto \Li 2(u,y)$ is affine for all $y \in \Y$, enumerate the vertices of these regions as a finite representative set (since $\Li 2$ is polyhedral).  
By Theorem~\ref{thm:polyhedral-embeds-discrete}, $\Li 2$ embeds its restriction to these vertices.
We can then study the property elicited by the embedded loss and compare it to the top-$k$ property to understand which distributional assumptions are needed for top-$k$ consistency.
This procedure is in Figure~\ref{fig:embedding-procedure}.

In this section, we take $\bar u$ to be the average of the top $k$ elements of $u$, $\bar u = \frac{\sigma_k(u)}{k}$.
Moreover, we denote $\bar{u}_{-i} = \frac{1}{k-1} (\sigma_k(u) - u_{[i]})$ be the averages of the first $k$ sorted elements of $u$ and the average of the first $k-1$ sorted elements of $u$ besides the $i^{th}$ element, respectively.
When $u \in \U_2$ as defined below, $\bar u_{-i}$ is the average of the top-$k$ elements of $u$ if $i \not \in T_k(u)$ and the top $k-1$ of $u_{\setminus i}$ otherwise.

\begin{figure}[t]
    \centering
    \includegraphics[width=0.9\linewidth]{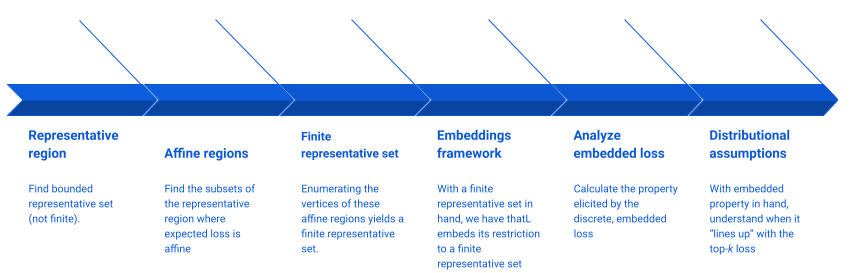}
    \caption{The general embedding procedure used to analyze $\Li{2}$, $\Li{3}$, and $\Li{4}$.}
    \label{fig:embedding-procedure}
\end{figure}

\subsection{The Bounded Representative Region}\label{sec:assumptions}
We initially bound our report set with upper and lower bounds, and show the restricted set is representative. 
We first observe that $\Li{2}$ is invariant in the $\ones$ direction, which is necessary for our first restriction.

\begin{lemma}[Invariance in the $\ones$ direction]\label{lem:invar-ones-L2}
	$\Li{2}(u,y) = \Li{2}(u + \alpha \ones,y)$ for all $\alpha \in \reals$ and $y\in\Y$.
\end{lemma}
\begin{proof}
	\begin{align*}
	\Li{2}(u + \alpha \ones,y) &= \left(1 - (u_y + \alpha) + \frac{1}{k} \sum_{i=1}^k (u + \alpha \ones  - e_y)_{[i]} \right)_+\\
	&= \left(1 - (u_y + \alpha) +\frac{1}{k} \alpha k + \frac{1}{k} \sum_{i=1}^k (u - e_y)_{[i]} \right)_+\\
	&= \left(1 -u_y - \alpha + \alpha + \frac{1}{k} \sum_{i=1}^k (u - e_y)_{[i]} \right)_+\\
	&= \left(1 -u_y + \frac{1}{k} \sum_{i=1}^k (u - e_y)_{[i]} \right)_+\\
	&= \Li{2}(u,y)
	\end{align*}
\end{proof}

We now introduce our first restriction on reports and show it is representative.
\begin{lemma}\label{lem:extra-elts-zero}
	$\Rlo_2 := \{u \in \reals^n_+ \mid \|u\|_0 \leq k\} = \{u \in \reals^n \mid u_\kp = 0 = u_{[n]}\}$ is a representative set for $\Li{2}$.
\end{lemma}
\begin{proof}
	By Lemma~\ref{lem:invar-ones-L2}, we can fix $u_{[k+1]} = 0$ without loss of generality.
	We then have $u_{[i]} \leq 0$ for all $i \geq k+1$.
	Consider $u' = \max(u, \vec 0)$ such that $u'_{[i]} = u_{[i]}$ for all $i$ such that $u_{[i]} \geq 0$, and $u'_{[i]} = 0$ otherwise.
	Observe that $u$ and $u'$ have the same ordering on their elements and $u' \in \Rlo_2$ by construction.
	We want to show that $\Li{2}(u,y) \geq \Li{2}(u',y)$ for all $y \in \Y$, and representativeness of $\Rlo_2$ follows.
	
	If $u_y \leq u_\kp = 0$, then 
	\begin{align*}
	\Li{2}(u,y) &= \left(1 - u_y + \frac 1 k \sum_{i=1}^k (u - e_y)_{[i]}\right)_+ &  \\
	&\geq \left(1 + \frac {1}{k} \sum_{i=1}^k (u' - e_y)_{[i]}\right)_+ & \text{top-$k+1$ elements of $u$ and $u'$ are the same and $u_y \leq 0$}\\
	&= \Li{2}(u',y)~.~
	\end{align*}
	The inequality comes from the equality of the first $k+1$ sorted elements of $u$ and $u'$, combined with setting $u_y \leq u'_y = 0$ in this case.
	
	Now, if $u_y \geq u_\kp$ and $u_y \geq 1$, observe that $\bar u = \bar u'$.
	We then have 
	\begin{align*}
	\Li{2}(u,y) &= \left(1 - u_y + \frac 1 k \sum_{i=1}^k (u - e_y)_{[i]}\right)_+ & \\
	&= \left(1 - u_y + \bar u - \frac 1 k\right)_+ & \text{substitution of summand by case}\\
	&= \left(1 - u'_y + \bar u' - \frac 1 k\right)_+ & \text{$\bar u = \bar u'$}\\
	&= \Li{2}(u',y)~.~
	\end{align*}
	
	Now suppose that $u_y \geq u_\kp$ and $u_y \in [0,1)$; as the top-$k+1$ elements of $u$ and $u'$ are the same, observe $(u - e_y)_\k = u_\kp = 0 = (u' - e_y)_\k$.
	\begin{align*}
	\Li{2}(u,y) &= \left(1 - u_y + \frac 1 k \sum_{i=1}^{k-1} u_{[i]} + 0 \right)_+ & \text{case} \\
	&= \left(1 - u'_y + \frac 1 k \sum_{i=1}^{k-1} u'_{[i]}\right)_+ & \text{$u$ matches $u'$ on top $k+1$ elements and $u_y = u'_y$} \\
	&= \left(1 - u'_y + \frac 1 k \sum_{i=1}^{k} (u' - e_y)_{[i]}\right)_+ & \\
	&= \Li{2}(u',y)~.~
	\end{align*}

	Since $\Li{2}(u,y) \geq \Li{2}(u',y)$ for all $y \in \Y$, this also holds for the expected loss for all $p \in \simplex$.
	Thus, $\Rlo_2$ is representative.
\end{proof}

It follows from construction of $\Rlo_2$ that we can take $u_\kp = \ldots, = u_{[n]} = 0$ and still have a representative set.

Now, consider the set $\Rhi_2 := \{u \in \reals^n_+ \mid u_i \leq \bar u_{-i} +1 \,\,\,\, \forall i \in [n]\}$. 
While $\Rlo_2$ gives a lower bound on a representative region, $\Rhi_2$ gives an upper bound.

\begin{lemma}\label{lem:top-threshold}
The set $\U_2 := \Rlo_2 \cap \Rhi_2$ is representative.
\end{lemma}
\begin{proof}
  Since we have already proven $\Rlo_2$ is representative in Lemma~\ref{lem:extra-elts-zero}, suppose $u \in \Rlo_2$.
  Any $u \not \in \Rhi_2$ must have some element $i \in [n]$ such that $u_i > \bar u_{-i} + 1$.
  Consider $u'$ as follows: for all $y \in [n]$ such that $u_y > \bar u_{-y} + 1$, reassign such a $u'_y := \bar u_{-y} +1$.
  We proceed in two cases, showing below that $\Li{2}(u',y) = \Li{2}(u,y) = 0$ due to the positive part operator.
  In the second case, $\Li{2}(u',j) \leq \Li{2}(u,j)$ for all $j \neq y$ as $\bar u < \bar u'$.
  Moreover, $u' \in \Rhi_2$ by construction.

First, we consider when the outcome $y$ is the modified element of $u$.
We write $u_y = \bar u_{-y} + 1 + \epsilon$ for some $\epsilon > 0$ and $u'_y = \bar u_{-y} + 1$, with $u_j = u'_j$ for all $j \neq y$. 
\begin{align*}
\Li{2}(u,y) &= \left(1 - (\bar u_{-y} + \epsilon + 1) + \frac 1 k \left(\sum_{j=1, j \neq y}^k u_j + \bar u_{-y} + \epsilon\right) \right)_+\\ 
&= \left(1 - (\bar u_{-y} + \epsilon + 1) + \frac 1 k \left((k-1) \bar u_{-y} + \bar u_{-y} + \epsilon\right) \right)_+\\
&= \left(-\bar u_{-y} - \epsilon + \bar u_{-y} + \frac \epsilon k \right)_+ \\
&= (-\frac{k-1}{k}\epsilon)_+\\
&= 0%
\end{align*}
When $\epsilon = 0$, we recover $u'$, in which case we observe the same result from $\Li 2(u',y) = (-\frac {k-1}{k} 0)_+ = 0$.
Thus, the losses are equal on the outcome $y$.

Now, let us consider $z\neq y$.
Since $u \in \Rlo_2$, we have $\bar u \geq 0$ and $\bar u_{-i} \geq 0$ for any $i \in [n]$.
Therefore, if $u_y > \bar u_{-y} + 1$, then we have $u_y > u_\kp$ as $u_y + 1 \geq 1 > 0 = u_\kp$.
Now, for outcome $z \neq y$ (with $u_z \leq \bar u_{-z} +1$, and therefore $u_z = u'_z$), we have 
\begin{align*}
\Li{2}(u,z) &= \left( 1 - u_z + \frac 1 k \sum_{i=1}^k (u-e_z)_{[i]} \right)_+ & \\
	&= \left( 1 - u'_z + \frac 1 k \sum_{i=1}^k (u'-e_z)_{[i]} + \frac {\epsilon}{k}\right)_+ & \text{$u'_y$ is in the top $k$ elements of $(u' - e_z)$ and $u_z = u'_z$}\\
	&\geq \left( 1 - u'_z + \frac 1 k \sum_{i=1}^k (u'-e_z)_{[i]}\right)_+ & \text{Since $\epsilon > 0$}\\
	&= \Li{2}(u',z)
\end{align*}
If there is more than one index $y$ such that $u_y > \bar u_{-y} + 1$, we can repeat this procedure in decsending order so the result holds.

Therefore, if $u \in \argmin_{r}\Li{2}(r,y)$, then so is some $u' \in \U_2$ for each $y \in \Y$, and we can say the same of the expected loss $\E_p \Li{2}(u,\cdot)$ for all $p \in \simplex$.
Thus, $\argmin_u \E_p \Li{2}(u,\cdot) \cap \U_2$ is nonempty for all $p \in \simplex$ and therefore $\U_2$ is representative.
\end{proof}

\paragraph{Re-writing the surrogate without the positive part operator.}
For any $u \in \U_2$, we can rewrite $\Li 2|_{\U_2}(u,y) = \Li{2}(u,y) =  1 - u_y + \frac 1 k \sum_{i = 1}^k(u - e_y)_{[i]}$, removing the positive part operator, as the term inside is always nonnegative for $u \in \U_2$.
This allows us to re-write the loss as follows: 
\begin{align}
\Li{2}(u,y) &= 1 - u_y + \bar u - \frac 1 k \min(u_y, 1)~.\label{eq:simplified-psi-2}
\end{align}
Moreover, we can evaluate the expected loss
\begin{align}
    \E_p \Li{2}(u,\cdot) &= \sum_y p_y \left( 1 - u_y + \bar u - \frac 1 k \min(u_y, 1) \right) \nonumber \\
    &= \sum_y p_y (1 + \bar u)  - \sum_y p_y u_y - \sum_y p_y \frac 1 k \min(u_y, 1) \nonumber \\
    &= 1  - \inprod{p}{u} + \bar u - \frac 1 k \inprod{p}{\min(u, \ones)}~.\label{eq:simplified-ex-psi-2}
\end{align}

\subsection{Affine Regions and a Finite Representative Set}
Since the loss $\Li{2}$ is polyhedral, it has a finite set of minimizers~\citep[Lemma 2]{finocchiaro2019embedding}.
Upon finding a finite representative set $\Ri{2} \subseteq \U_2$, we can apply Theorem~\ref{thm:polyhedral-embeds-discrete}(2) and study the property elicited by $\Li{2}|_{\Ri{2}}$ via embeddings, and how it compares to the top-$k$ property $\gamma_k := \prop{\lk}$ under the argmax link. 

As we showed $\U_2$ is representative in Lemma~\ref{lem:top-threshold}, consider the following set 
\begin{align*}
\Ri{2} &:= \left\{ \frac{|M| + k -1}{k - |H|} \mathbf{1}_H + \mathbf{1}_M : H, M \subset [n], H\cap M = \emptyset, |H| + |M| \leq k, |H| < k \right\}~.~
\end{align*}

We will show that $\Ri{2}$ enumerates the vertices of the regions where the function $u\mapsto \E_p \Li{2}(u,\cdot)$ must be affine, regardless of $p \in \simplex$.
Moreover, is the expected loss is polyhedral, it minimized on at least one face of these affine regions; since each face contains at least one vertex in $\Ri 2$, we will conclude $\Ri{2}$ is representative. 
 
\begin{lemma}\label{lem:topk-affine-regions}
	Fix a set $T \subseteq [n]$ such that $|T| = k$ and any $S \subseteq T$.
	Then $\Li{2}(\cdot,y)$ is affine on the set $A^{T,S} := \{u \in \U_2 \mid T \in T_k(u) \wedge u_y \in [0,1] \forall y \in S \wedge u_y \in [1,1 + \frac 1 {k-1} \sum_{j \in T_k(u) \setminus \{i\}} u_j] \forall y \in T \setminus S \}$ for all $y \in \Y$.
\end{lemma}
\begin{proof}
	First, observe the that $u \mapsto \Li 2(u,y)$ is affine in the first two terms of eq.~\eqref{eq:simplified-psi-2} for all $u \in \U_2$, and nonlinearity is only introduced in the last two terms of eq.~\eqref{eq:simplified-ex-psi-2}.
	Fix any set $T \subseteq [n]$ of size $k$.
	We denote by $A^T := \{u \in \U_2 \mid \Tk(u) = T\}$ as the set of $u$ whose top $k$ elements are exactly the elements of $T$.
	Observe that $u \mapsto \bar u$ is affine in $A^T$ for each $T$ since $\bar u$ is the sum of the top $k$ elements of $u$, regardless of their relative order.
	
	Now, since $u \in \U_2 \supseteq \Rlo_2$, we have $u_\kp = \ldots = u_{[n]} = 0$, we impose $k$ constraints constructing $\Rlo_2$ given by $0 \leq u_{i}$ for $i \in T$.
	Moreover, there are $k$ constraints constructing $\Rhi_2$, given by $u_{i} \leq 1 + \bar u_{-i} = 1 + \frac{1}{k-1}\sum_{j \in T_k(u) \setminus \{i\}} u_j$ for all $i \in T$.
	
	We now consider affineness of $u \mapsto \frac 1 k \min(u, \ones)$, where there is a ``switch'' of affine regions at $u_y = 1$ for each $y$.
	For a fixed set $T$ and $u \in A^T$, consider any $S \subseteq T$.
	Construct the region $A^{T,S} = \{u \in A^T \mid u_y \in [0,1] \forall y \in S \wedge u_y \in [1,1 + \frac 1 {k-1} \sum_{j \in T_k(u) \setminus \{i\}} u_j] \forall y \in T \setminus S \}$.
	Observe that $A^{T,S} \subseteq A^T$.
	
	Since $A^{T,S} \subseteq A^T$, we know that $u \mapsto \bar u$ is affine on $A^{T,S}$, and construct $A^{T,S}$ so that $u \mapsto -\frac 1 k \min(u, \ones)$ is affine on $A^{T,S}$.
	Therefore, $u \mapsto L(u,y)$ is affine on $A^{T,S}$ for all $y \in \Y$, $T$ of size $k$, and $S \subseteq T$ as it is the sum of affine functions. 
\end{proof}

As vertices of the $A^{T,S}$ regions are formed by the intersections of these affine regions, we can now enumerate the vertices of the $A^{T,S}$ regions with $\Ri{2}$. 

\renewcommand{\vert}{\mathrm{vert}}
\newcommand{\V}{\mathcal{V}}
\begin{lemma}\label{claim:R-vertices}
	Let $\vert(A^{T,S})$ be the vertices of the the region $A^{T,S}$, and let $\V := \cup_{T,S \mid |T| = k, S \subseteq T} \vert(A^{T,S})$.
	Then $\V \subseteq \Ri{2}$.
\end{lemma}
\begin{proof}
	By a corollary of Lemma~\ref{lem:topk-affine-regions}, the function $u \mapsto \E_p \Li 2(u,\cdot)$ is affine on $A^{T,S}$ for any $T \subseteq [n]$ of size $k$ and $S \subseteq T$ for all $p \in \simplex$. 
	We now proceed to compute $\V$ by finding $k$ equalities imposed on $A^{T,S}$~\citep{grunbaum1967convex}.
	Vertices of each $A^{T,S}$ region are formed by the intersection of $n$ hyperplanes technically, but with $T$ fixed, the other $n-k$ come from the requirement $u_i = 0 \forall i \in [n] \setminus T$.
	
	Fix $T,S$ such that $|T| = k$ and $S \subseteq T$.
	We then have vertices at each of these $2^k$ possible equalities, given by the following constraints.
	\begin{align*}
	\forall i \in S, \,\, 0 &\leq u_i \leq 1 \\
	\forall i \in T \setminus S, \,\, 1 &\leq u_i \leq 1 + \frac 1 {k-1} \sum_{j \in T_k(u) \setminus \{i\}} u_j 
	\end{align*}
	Iterating over each of these $2^k$ inequalities, we see that vertices are generated at points $0,1$ or some constant $c_{S,T} \geq 1$ for each choice of inequalities for $T$ and $S$.
	
	It suffices to show that for each $T$ and $S$ as above, there is a $H$ and $M$ satisfying the requirements of $\Ri{2}$.  
	In particular, we take $M = \{i \in S \mid u_i = 1\}$, and $H := \{i \in S \mid u_i > 1\}$.
	By construction, we have $H \cap M = \emptyset$, and $|H| + |M| \leq k$.
	Thus, every $v \in \V$ is contained in $\Ri{2}$.
\end{proof}

\begin{corollary}
	$\Ri{2}$ is a finite representative set for $\Li{2}$.
\end{corollary}

\subsection{The Loss Embedded by $\Li 2$}
\begin{corollary}
	$\Li{2}$ embeds $\Li{2}|_{\Ri{2}}$.
\end{corollary}

We can now evaluate the restricted function and obtain it in the form of a loss matrix.
\begin{align}
\Li{2}|_{\Ri{2}}(r,y) &= \begin{cases}
0 & r_y = \bar r_{-y} + 1\\
\bar r - \frac 1 k & r_y = 1\\
1 + \bar r & r_y = 0
\end{cases}~.~
\end{align}

We can equivalently relabel the reports in $\Ri{2}$ via a bijection $\Phi$ designating $y$ as an element of $H$ if $u_y > 1$, and $y$ is an element of $M$ if $u_y = 1$.
\begin{align*}
\hat \ell_2((H,M), y) &= \begin{cases}
    0 & y \in H \\
     \frac{|H| + |M| - 1}{k-|H|}
     & y \in M\\
    \left( \frac{|H| + |M| - 1}{k-|H|} \right) + \frac{k+1}{k}
     & \text{otherwise}
    \end{cases}
\end{align*}

\subsection{The Property Elicited by the Embedded Loss}
The next natural question is to consider is whether or not $(\Li{2}, \psi_k)$ is calibrated with respect to $\ell_k$.
In order to answer this, we necessarily need to understand something about $\Gamma := \prop{\Li{2}}$, which we will study through $\gamma := \prop{\hat \ell_2}$.

In the previous subsection, we saw the construction of ``high'' ($H$), ``meduim'' ($M$), and ``low'' ($L := [n] \setminus (H \cup M)$) bins for the elements of $\Ri{2}$ via the bijection $\Phi$.
However, because of the nature of $\U_2$, there is a dependence of multiple coordinates for an optimal report of $\Li{2}$.
That is, for a fixed probability distribution $p \in \simplex$, there may be coordinates $i \in [n]$ with ``enough'' weight for $i \in H\cup M$, but there is sometimes a benefit in expected loss for this surrogate by artificially bumping up from the ``low'' to ``middle'' bin when possible because doing so cranks up the constant on the ``high'' reports, yielding better expected loss.
That is, sometimes an algorithm is confident enough in its ``high'' labels that it is optimal to take an additional expected loss on some ``lower'' labels. 

Next, we characterize the distributions $p$ such that $\E_p \hat \ell_2 ((H \cup\{i\}, M), \cdot) \leq \E_p \hat \ell_2 ((H, M \cup \{i\}), \cdot)$.

\begin{lemma}\label{lem:H-better-M}
    Fix some $(H,M) \in \Phi(\Ri 2)$ and consider any index $i \in [n] \setminus (H \cup M)$.
    Consider $u \in \Phi(\Ri{2})$ such that $u = (H \cup \{i\}, M)$ and $u' = (H, M \cup \{i\})$.
    Then $\E_p \hat \ell_2(u,\cdot) \leq \E_p\hat \ell_2(u',\cdot)$ if and only if 
    $p_i \geq (1 - \sigma_H(p))(\frac {1}{k-h})$.
\end{lemma}
\begin{proof}
Let $h = |H|$ and $m = |M|$.
\begin{align*}
    \E_p \hat \ell_2(u,\cdot) &\leq \E_p \hat \ell_2 (u',\cdot)\\ 
    \left(1 - (\sigma_H(p) + p_i)\right) \left(\frac {k(h+1+m)} {k-h-1}\right) &\leq \left(1 - \sigma_H(p)\right) \left(\frac {k(h+1+m)} {k-h}\right) \\
    \frac{\left( k(1 - \sigma_H(p)) \right)}{(k-h-1)(k-h)} &\leq p_i \frac{k}{k-h-1} \\
    \frac{1 - \sigma_H(p)}{k-h} &\leq p_i~.
\end{align*}
\end{proof}
Observe that for $H = \emptyset$, this inequality becomes $p_i \geq \frac 1 k$.
Now, we characterize the distributions $p$ such that $\E_p \hat \ell_2 ((H, M \cup\{i\}), \cdot) \leq \E_p \hat \ell_2 ((H, M), \cdot)$.

\subsubsection{$M \succeq_i L$}
\begin{lemma}\label{lem:M-better-L}
    Fix $(H,M) \in \Phi(\Ri 2)$ and consider any index $i \in [n] \setminus (H \cup M)$.
    Consider $u \in \Phi(\Ri{2})$ such that $(H, M \cup \{i\})$ and $u' = (H,M)$.
    Then $\E_p \hat \ell_2 (u,\cdot) \leq \E_p \hat \ell_2(u',\cdot)$ if and only if 
    $p_i \geq (\frac h k - \sigma_H(p))(\frac {k}{(k-h)(k+1)}) + \frac 1 {k+1}$.
\end{lemma}
\begin{proof}
Let $h = |H|$ and $m = |M|$.
\begin{align*}
\E_p \hat \ell_2 (u,\cdot) &\leq \E_p \hat \ell_2 (u',\cdot)\\
(1 - \sigma_H(p))\tfrac{k(h+m+1)}{k-h} + \tfrac{(k+1)(1 - \sigma_H(p) - \sigma_M(p) - p_i)}{k} &\leq (1 - \sigma_H(p))\tfrac{k(h+m)}{k-h} + \tfrac{(k+1)(1 - \sigma_H(p) - \sigma_M(p))}{k}\\
\frac{k(1 - \sigma_H(p))}{k-h} &\leq \frac{k+1}{k} p_i \\
\frac{1 - \sigma_H(p)}{(k+1)(k-h)} &\leq p_i~.
\end{align*}
\end{proof}

Lemmas~\ref{lem:H-better-M} and \ref{lem:M-better-L} now provide testable conditions to yield $\prop{\hat \ell_2}$ as $\Ri 2$ is finite.
Now let us consider how one wants to assign indices to each of these three bins.

Consider first that we can calculate the set of indices that should be designated in $H$.
\begin{equation}
    h^*(p) = \max \left\{i \in \{0,\ldots, k-1\} \mid  p_\i > \frac{1 - \sum_{j=1}^{i-1}p_\j}{k-(i-1)} \right\}%
\end{equation}
Now, let us consider $p_{h^*(p)} := \sigma_{h^*(p)}(p) = \sum_{j=1}^{h^*(p)} p_\j$ to determine which elements of $p$ should be designated in $M$. 

\begin{equation}
    m^*(p) = \max \{j \in \{0,\ldots, k\} \mid  p_\j > \frac{1 - p_{h^*(p)}}{(k+1)(k-h^*(p))} \}
\end{equation}

\subsection{Characterizing Consistency of $\Li 2$ with Respect to $\lk$}

We have consistency via the canonical argmax link $\psi_k$ when the optimal surrogate reports $u$ have $u_\k > 0 = u_\kp$, since its top-$k$ set is unique.
For intuition, consider that inconsistency means that any sequence of reports $\{u_i\}$ approaching the $\Li 2$ optimum and applying the link (e.g., $\{r_i\} = \{\psi(u_i)\}$ approaches the $\Li 2|_{\Ri 2}$ optimum; equivalently, approaching the $\Li 2$ optimum implies that $(H_i. M_i) = \{\Phi(r_i)\}$ approaches the $\hat \ell_2$ optimum.

Consider some distributions $p,p'$ such that $\Tk(p) \neq \Tk(p')$ but $(H,M) \in \prop{\hat \ell_2}(p) = \prop{\hat \ell_2}(p')$ with $|H \cup M| < k$.
As the link must be deterministic, given $u = \Phi^{-1}((H,M))$, the link must choose some ordering over the elements $S \subseteq [n]$ such that $u_i = 0$ for all $i \in S$.
Even if this ordering aligns with $\Tk(p)$, it will not align with $\Tk(p')$ as they are not equal; hence the ambiguity in $\Tk(u)$ makes it impossible for consistency to hold at both $p$ and $p'$.
Thus, we will only have consistency guaranteed at distributions $p \in \simplex$ such that there is a value $u \in \prop{\Li 2}(p)$ with $\Tk(u)$ unambiguous.
The distributions where this condition holds are exactly the $p$ for which $m^*(p) = k$.

\begin{lemma}
Let $L:\reals^n\times \Y\to\reals_+$ be a polyhedral loss which embeds $\hat\ell:\hat\R\times\Y\to\reals_+$.
Let $\ell:\R\times\Y\to\reals_+$ be a target loss.
Let $\P \subseteq \simplex$.
Let $\hat\gamma = \prop{\hat\ell}$, $\gamma=\prop\ell$.
If for all $\hat r\in\hat \R$, there exists some $r\in\R$ such that $\{p\in\P : \hat r \in \hat\gamma(p)\} \subseteq \{p\in\P : r \in \gamma(p) \}$, then there exists a link function $\psi:\reals^n\to\R$ such that $(L,\psi)$ is calibrated with respect to $\ell$ on $\P$.
\label{lemma:restricted-consistency}
\end{lemma}
\begin{proof}
The proof is essentially the same as that of \citet[Theorem 8]{finocchiaro2022embedding}, but restricted to $\P \subseteq \simplex$. 
Since $L$ embeds $\hat \ell$, let $\hat \psi: \reals^n \to \hat \R$ be a link function such that $(L, \hat \psi)$ is calibrated with respect to $\hat \ell$ (Theorem \ref{thm:polyhedral-embeds-discrete}).
By \citet[Lemma 4]{finocchiaro2022embedding}, $\hat \ell$ indirectly elicits $\gamma$ for some link function $\psi^\R: \hat \R \to \R$.
Then for any $r \in \hat \R$ and $p \in \P$, $r \in \hat \gamma(p) \implies \psi^\R(r) \in \gamma(p)$.

Now, let $\psi = \psi^\R \circ \hat \psi: \R^n \to \R$.
We will show $(L, \psi)$ is calibrated with respect to $\ell$ on $\P$.
By the construction of $\psi$, for any $p \in \P$ and $u \in \reals^d$, if $\hat \psi(u) \in \hat \gamma(p)$, then $\psi(u) = \psi^\R(\hat \psi(u)) \in \gamma(p)$.
Similarly, if $\psi(u) \not\in \gamma(p)$, then $\hat \psi(u) \not \in \hat \gamma(p)$.
Therefore, 
\[ \{u \in \reals^n \mid \psi(u) \not \in \gamma(p)\} \subseteq \{u \in \reals^n \mid \hat\psi(u) \not \in \hat\gamma(p)\} ~.\]
Since $(L, \hat\psi)$ is calibrated with respect to $\hat \ell$, we obtain
\[ \inf_{u \in \reals^n:\psi(u) \not \in \gamma(p)} \inprod{p}{L(u)} \geq \inf_{u \in \reals^n:\hat\psi(u) \not \in \hat\gamma(p)} \inprod{p}{L(u)} > \inf_{u \in \reals^n} \inprod{p}{L(u)} ~,\]
so $(L, \psi)$ is calibrated with resepct to $\ell$ on $\P$.
\end{proof}

\psitwoconsistent*

\section{Additional Derivations for $\Li{3}$} \label{appendix:L3}

Recall that we have
\[ \Lthree\]

We follow the same general procedure as \S~\ref{appendix:L2}; see Figure~\ref{fig:embedding-procedure} for an outline.

\subsection{Finding a Representative Region}
As with $\Li{2}$, we can show that $\Li{3}$ is invariant in the ones direction.
\begin{lemma}[Invariance in the $\ones$ direction]\label{lem:invar-ones}
$\Li{3}(u, y) = \Li{3}(u + \alpha \ones, y)$ for all $\alpha \in \reals$.
\end{lemma}
\begin{proof}
\begin{align*}
    \Li{3}(u + \alpha \ones, y) &= \frac 1 k \sum_{i=1}^k \left[ 1 - (u_y + \alpha) + (u + \alpha \ones - e_y)_{[i]}) \right]_+ \\
    &= \frac 1 k \sum_{i=1}^k \left[ 1 - u_y - \alpha + (u - e_y)_{[i]} + \alpha \right]_+ \\
    &= \frac 1 k \sum_{i=1}^k \left[ 1 - u_y + (u - e_y)_{[i]} \right]_+ \\
    &= \Li{3} (u,y)~.
    \qedhere
\end{align*}
\end{proof}

As before, we can then set $u_{[k+1]} = 0$ without loss of generality, and show that we can restrict to the representative set $\Rlo_3 := \{u \in \reals^{n}_+ \mid \|u\|_0 \leq k\}$ (Lemma~\ref{lem:Rlo-rep}).
Throughout, let $T(u) \in \Tk(u)$ be some choice of top-$k$ elements of $u$ so that $|T(u)| = k$.
For a fixed choice $T(u)$, we additionally consider $V_y(u) := \{i \in T(u) \mid u_i > u_y -1\} \setminus \{y\}$. 

\begin{lemma}[$\Rlo_3$ is representative for $\Li{3}$]\label{lem:Rlo-rep}
Consider $u \in \reals^n_+$ such that $u_{[k+1]} = 0$, and $u'$ such that $u'_\i = u_\i$ for $i \leq k+1$, and $u'_\i = 0$ otherwise.
For all $y \in \Y$, $\Li{3}(u,y) \geq \Li{3}(u',y)$.
\end{lemma}
\begin{proof}
Observe that there is a choice of $T$ such that $T(u) = T(u')$; we proceed with this choice, though any other choice of $T(u') \in \Tk(u)$ results in the same loss values.
Consider two cases: first, if $y \in T(u)$, and then if $y \not \in T(u)$.

Case 1: $y \in T(u)$ follows trivially since the elements being summed over are equal (e.g., $u_i = u'_i \forall i \in T(u) = T(u')$), so the losses are equal.

Case 2: $y \not \in T(u)$
\begin{align*}
    \Li{3}(u,y) 
    &= \frac 1 k \sum_{i \in T(u)} \left[ 1 - u_y + (u - e_y)_\i \right]_+\\
    &= \frac 1 k \sum_{i \in T(u)} \left[ 1 - u_y + u_i \right]_+\\
    &= \frac 1 k \sum_{i \in T(u')} \left[ 1 - u_y + u'_i \right]_+\\
    &\geq \frac 1 k \sum_{i \in T(u')} \left[ 1 + u'_i \right]_+\\
    &= \Li{3}(u',y)~.
\end{align*}

By Lemma~\ref{lem:invar-ones}, $u$ is invariant in the ones direction, so without loss of generality we can set $u_{[k+1]} = 0$.
The cases above show we can set $u_{[j]} = 0$ for any $j > k + 1$ without increasing the loss on any outcome.
Together, these results imply that $\Rlo_3 = \{u \in \reals^n_+ \mid \|u\|_0 \leq k\}$ is representative.
\end{proof}

We continue towards a finite representative set, showing that each element of $u$ should be no more than $1$ greater than the next lowest element 
\begin{lemma}\label{lem:match-Ty-S}
Consider $u \in \Rlo_3$ and $i \in \{1,\ldots,k\}$ such that $u_{[i]} = u_{[i+1]} + 1+ \varepsilon$ for some $\varepsilon > 0$.
Take $u'$ such that $u'_{[j]} = u_{[j]} - \varepsilon$ for all $j \in \{1, \ldots, i\}$.
Then there exists a choice of $T$ such that $T(u) = T(u')$ and $(1 - u_y + u_j)_+ \geq (1 - u'_y + u'_j)_+$ for all $j \in V_y(u)$, so $V_y(u') \subseteq V_y(u)$.
\end{lemma}
\begin{proof}
First, observe $\Tk(u) = \Tk(u')$, as we are only shifting at most the top $k$ elements of $u$, and they are being shifted in a way that preserves them as the top-$k$.
Thus, by taking $T(u)$ to be a function of $\Tk(u)$, a choice of $T$ such that $T(u) = T(u')$ exists.

For any outcome $y \in \Y$ and index $j \in [n]$, there are four possible cases for the change in $u_y$ and $u_j$: (1) neither is modified (e.g., $u_y = u'_y$ and $u_j = u'_j$); (2) just $u_y$ is modified (e.g., $u_y = u'_y + \varepsilon$ and $u_j = u'_j$); (3) just $u_j$ is modified (e.g., $u_y = u'_y$ and $u_j = u'_j + \varepsilon$); and (4) both are modified (e.g., $u_y = u'_y + \varepsilon$ and $u_j = u'_j + \varepsilon$).  
Cases 1 and 4 are immediate, $(1 - u_y + u_j)_+ = (1 - u'_y + u'_j)_+$ by substitution.

Case 2: $u_y = u'_y + \varepsilon$, $u_j = u'_j$.
For this case to occur, $u_y \geq u_{\i}$ and $u_\i > u_j$. 
Therefore, $u_y > u_i+1 \geq u_j$, violating the construction of $V_y(u)$.

Case 3: $u_y = u'_y$, $u_j = u'_j + \varepsilon$.
By the case, we have $u_j > u'_j \geq u_{[i+1]} + 1 \geq u'_y = u_y$.
As $u'_j < u_j$, we immediately have $(1 - u_y + u_j)_+ \geq (1 - u'_y + u'_j)_+ \geq 0$.

\end{proof}

Let us denote the set $\Rhi_3 := \{u \in \reals^n \mid u_{[i+1]} \leq u_{[i]} \leq u_{[i+1]} +1\,\, \forall i \in (1,\ldots,k) \}$.
We now give a bounded representative set for $\Li{3}$.

\begin{lemma}\label{claim:rep-bounded-plus-one}
The set $\U_3 := \Rhi_3 \cap \Rlo_3$ is representative for $\Li{3}$. 
\end{lemma}
\begin{proof}

Fix any $u \in \Rlo_3$ such that for some $i \in \{1,\ldots,k\}$ and $\varepsilon > 0$, $u_{[i]} = u_{[i+1]} + 1+ \varepsilon$.
Take $u'$ such that $u'_{[j]} = u_{[j]} - \varepsilon$ for all $j \in \{1, \ldots, i\}$.
We want to show $\Li{3}(u,y) \geq \Li{3}(u',y)$ for all $y \in \Y$.

By construction of $V_y(u)$, we can write
\begin{align}\label{eq:L3-alt-form}
    \Li{3}(u,y) &= \sum_{j \in V_y(u)} (1 - u_y + u_j)~.~
\end{align}
Moreover, we have the existence of a $T$ such that $T(u) = T(u')$ and $V_y(u) \supseteq V_y(u')$ by Lemma~\ref{lem:match-Ty-S}.

We can consider 3 cases for any $j \in V_y(u)$: (1) $u'_y = u_y$ and $u'_j = u_j$; (2) $u'_y + \varepsilon = u_y$ and $u'_j + \varepsilon = u_j$; and (3) $u'_y = u_y$ and $u'_j + \varepsilon = u_j$.
For cases (1) and (2), we immediately have $(1 - u_y + u_j) = (1 - u'_y + u'_j)$, and for (3), we have $(1 - u_y + u_j) = (1 - u'_y + u'_j + \varepsilon) > (1 - u'_y + u'_j)$.

\begin{align*}
    \Li{3}(u,y) &= \sum_{j \in V_y(u)} (1 - u_y + u_j) & \\
    &\geq \sum_{j \in V_y(u')} (1 - u_y + u_j) & \text{Since $V_y(u) \supseteq V_u(u')$}\\
    &\geq \sum_{j \in V_y(u')} (1 - u'_y + u'_j) & \text{By substitution}\\
    &= \Li{3}(u',y)~.
\end{align*}
As this is true for all $y \in \Y$, we have $\E_p \Li{3}(u,\cdot) \geq \E_p \Li{3}(u',\cdot)$ for all $p \in \simplex$, yielding the result.
\end{proof}

\subsection{Characterizing Affineness}

Furthermore, we can show that $u \mapsto\E_p \Li{3}(u,\cdot)$ is affine on the following regions for all $p \in \simplex$.

\begin{lemma}\label{lem:affine-on-permutation-plus-one}
Fix a set $T \subseteq n$ such that $|T| = k$ and the set $\vec V = \{V_y \mid V_y \subseteq T, y \in \Y \}$.
\begin{align*}
    A^{T,\vec V} &= \{u \in \U_3 \mid T \in \topkset(u) \wedge V_y = V_y(u), \!\,\, \forall y \in \Y \}~.%
\end{align*}
Then $u \mapsto \E_p \Li{3}(u, \cdot)$ is affine on each $A^{T,\vec V}$ for all $p \in \simplex$. 
\end{lemma}
\begin{proof}
Nonaffineness in $u \mapsto \Li{3}(u,y)$ for any $y \in \Y$ is imposed where there is a change in $T(\cdot)$ or in $V_y(\cdot)$ since we can write $\Li 3(u,y) = \sum_{i \in V_y(u)} (1 - u_y + u_i)$ as in eq.~\eqref{eq:L3-alt-form}.
As non-affineness is only introduced in the terms of the summand, we construct $A^{T,\vec V}$ so that $T(u) \in \topkset(u)$ and $V_y(u)$ is constant on $A^{T,\vec V}$, and thus the terms of the summand are constant on $A^{T, \vec V}$.
Therefore, $u \mapsto \E_p \Li 3(u,\cdot)$ is affine on $A^{T, \vec V}$ for all $p \in \simplex$.
\end{proof}

\subsection{Constructing a Finite Representative Set}

When constructing a finite representative set, it is sufficient to consider the vertices of these affine regions; thus, Lemma~\ref{lem:affine-on-permutation-plus-one} yields a finite representative set as follows.
\begin{corollary}
$\Ri{3} := \U_3 \cap \mathbb{Z}_k^n$ is a finite representative for $\Li{3}$.
\end{corollary}
Thus, we can think of the loss $\Li{3}|_{\Ri{3}}$ as taking in as predictions an ordered partition of size at most $k$ partitions.
As with $\Li 2$, we can relabel the elements of $\Ri 3$ via some bijection $\Phi$; in particular, we consider a bijection to ordered partitions as follows. 
\newcommand{\Q}{\mathcal{Q}}
Let $\Q =\{(Q_0, Q_1, \ldots, Q_s) \mid s \leq k, Q_i \cap Q_j = \emptyset \forall i \neq j, |Q_s \cup \ldots \cup Q_1| \leq k\}$.
Let $\Phi : \Ri 3 \to \Q$ be the bijection $u \mapsto (\{i \in [n] \mid u_i = 0\}, \{i \in [n] \mid u_i = 1\}, \ldots, \{i \in [n] \mid u_i = s\})$.
Then we can denote $\hat \ell_3$ such that $\Li 3(u,y) = \hat \ell_3(\Phi(u),y)$ for all $u \in \Ri 3$.
\begin{align}
\hat \ell_3(Q, y) &=
\begin{cases}
\frac 1 k \left(|Q_j| - 1 + \sum_{i > j} |Q_i| (i-j +1)\right) & j > 0\\
\frac 1 k \left(\sum_{i=1}^s |Q_i|(i+1) \right) & j = 0
\end{cases}~,
\end{align}
where $y \in Q_j$.

\subsection{Analyzing the Loss Embedded by $\Li{3}$: Characterizing Consistency}
Now that we have the finite representative set $\Ri{3}$ for $\Li{3}$
, we can characterize 
the property elicited by $\Li 3$.

\begin{lemma}\label{lem:bump-one-up-L3}
Fix $u \in {\Ri{3}}$ with $u_{[k]} =1$, and consider $u' \in {\Ri{3}}$ such that $u'_{[k]} = 0$ and $u_\i = u'_\i$ for all $i \in \{1, \ldots, k-1\}$.
Then $\E_p \Li{3}(u,\cdot) \geq \E_p \Li{3}(u',\cdot) \iff \frac{\sum_{i = k+1}^n p_\i}{k-1} \geq p_\k$.
\end{lemma}
\begin{proof}
First, observe that we are not changing the relative order of elements of $u$ and $u'$, so there is a choice $T \in \Tk$ such that $T(u) = T(u')$, and for each $y$, the loss is positive on the same set of indices.
\begin{align*}
    \E_p \Li{3}(u,\cdot) &\geq \E_p \Li{3}(u',\cdot)\\
    \sum_{y \neq \k} p_y \!\sum_{i \in V_y(u)} (1-u_y + u_i) + p_\k \!\!\sum_{i \in T(u) \setminus k}(1-u_\k+ u_i) &\geq \!\sum_{y \neq \k} p_y \!\sum_{i \in V_y(u)} (1-u'_y + u'_i) + p_\k \!\!\sum_{i \in T(u') \setminus k}(1-u'_\k+ u'_i)\\
    \sum_{y : p_y < p_\k} p_y - p_\k (k-1) &\geq 0 \\
    \sum_{y : p_y < p_\k} p_y &\geq p_\k (k-1) \\
    \frac{\sum_{y : p_y < p_\k} p_y}{k-1} &\geq p_\k \\
    \frac{\sum_{i=k+1}^n p_\i}{k-1} &\geq p_\k ~.
\end{align*}
The result follows.
\end{proof}

This result partially characterizes when it is better to keep the $k^{th}$ element of $u$ as $0$: when it only imposes change in that one element.
This is particularly important to characterize inconsistency for top-$k$; if $u_\k = u_{[k+1]} = 0$, then $|\Tk(u)| > 1$ for $u \in \prop{\Li{3}}$, so how to link $u$ is ambiguous.

However, we also need to understand when it is beneficial to bump every higher element up by $1$, which is given by the following result.
\begin{lemma}\label{lem:bump-all-up-L3}
Fix $u \in \R_3$ with $u_\j =0$, and consider $u'\in \Ri{3}$ such that $u'_\j = 1$ and $u_\i + 1 = u'_\i$ for all $i = 1, \ldots, j$.
Then $\Li{3}(u;p) \geq \Li{3}(u';p) \iff p_{[j+1]} \geq \frac 1 {k+1}$.
\end{lemma}
\begin{proof}
\begin{align*}
    \E_p \Li{3}(u,\cdot)  &\geq \E_p \Li{3}(u',\cdot) \\
    \sum_{i = j+1}^n p_\i (k+j) + \sum_{i=1}^j (j-1) &\geq \sum_{i = j+2}^n p_\i (k+j + 1) + \sum_{i=1}^{j+1} (j)\\
    p_{[j+1]}(k+j) &\geq \sum_{i=j+2}^n p_\i + p_{[j+1]}(j-1) + \sum_{i=1}^{j+1}p_\i\\
    p_{[j+1]}(k+j) &\geq 1 + p_{[j+1]}(j-1)\\
    p_{[j+1]} &\geq \frac 1 {k+1}~.
    \qedhere
\end{align*}
\end{proof}

Lemmas~\ref{lem:bump-one-up-L3} and~\ref{lem:bump-all-up-L3} together characterize the the distributions $p \in \simplex$ where the report $u \in \prop{\Li 3}(p) \cap \Ri 3$ has $u_\k > 0$.
Thus, for $u \in \prop{\Li 3}(p)$ for such distributions $p$, $u_\k > 0$ and therefore $|\Tk(u)| = 1$.
Applying Lemma \ref{lemma:restricted-consistency}, we obtain the desired consistency result.

\ellthreeconsistent*

\section{Additional Derivations for $\Li{4}$} \label{appendix:L4}

Recall that for a report $u \in \mathbb{R}^n$ and label $y \in \Y$,
\[\Li{4}(u,y)=\left(1-u_y+\frac{1}{k}\sum_{i=1}^k(u_{\backslash y})_{[i]}\right)_+ ~.\]
We again follow the procedure in \S~\ref{appendix:L2} to find a representative region for $\Li{4}$.
\subsection{Constructing a Bounded, Representative Region for $\Li{4}$}
To establish a bounded, representative region for $\Li{4},$ we must first show that $\Li{4}$ is invariant in the $\ones$ direction.
\begin{lemma}[Invariance in the $\ones$ direction]\label{lem:L4-invariance}
$\Li{4}(u,y)=\Li{4}(u+\alpha\ones,y)$ for all $\alpha\in\R.$
\end{lemma}
\begin{proof}
\begin{align*}
    \Li{4}(u+\alpha\ones,y)&=(1-(u_y+\alpha)+\frac{1}{k}\sum_{i=1}^k((u+\alpha\ones)_{\backslash y})_{[i]})_+\\
    &=(1-u_y-\alpha+\frac{1}{k}\sum_{i=1}^k(u_{\backslash y}+\alpha\ones_{\backslash y})_{[i]})_+\\
    &=(1-u_y-\alpha+\frac{1}{k}\sum_{i=1}^k(u_{\backslash y})_{[i]} + \frac{1}{k}k\alpha)_+\\
    &=(1-u_y-\alpha+\frac{1}{k}\sum_{i=1}^k(u_{\backslash y})_{[i]} + \alpha)_+\\
    &=(1-u_y+\frac{1}{k}\sum_{i=1}^k(u_{\backslash y})_{[i]})_+\\
    &=\Li{4}(u,y)
    \qedhere
\end{align*}
\end{proof}
Let the sets $\Rlo_4$ and $\Rhi_4$ be defined as follows:
\begin{itemize}
    \item $\Rlo_4 = \{u\in \mathbb{R}^n_{+} \mid ||u||_0\leq k\}$
    \item $\Rhi_4 = \{u\in\mathbb{R}^n_+ \mid u_y\leq 1+\frac{1}{k}\sum_{i=1}^k(u_{\backslash y})_{[i]} \quad\forall y\in \Y\}$~.
\end{itemize}
We will show in Theorem~\ref{thm:U4-representative} that the intersection $\U_4 := \Rlo_4 \cap \Rhi_4$ is representative.
\begin{lemma}\label{lem:Rlo4-representative}
$\Rlo_4$ is a representative set for $\Li{4}$.
\end{lemma}
\begin{proof}
Suppose that $u\in \mathbb{R}^n$ where $u_{[k+1]}=0$. 
By Lemma \ref{lem:L4-invariance}, $u_\kp = 0$ is without loss of generality. 
Let $u' = \max(u, \vec 0)$ be the element-wise max, which is in $\Rlo_4$ by construction.
It suffices to show that for all $y\in \Y, \Li{4}(u,y)\geq \Li{4}(u',y)$.

By construction, there is a set $S \subseteq \Y, |S| = k$ such that $S \in \Tk(u) \cap \Tk(u')$.
We proceed in two cases: if $y \in S$, and if $y \not \in S$.

\textbf{Case 1:} $y\in S$:\\
In this case, we have $u_y = u'_y \geq 0$.
\begin{align*}
    \Li{4}(u,y)&=(1-u_y+\frac{1}{k}\sum_{i=1}^k(u_\smy)_{[i]})_+\\
    &=(1-u'_y+\frac{1}{k}\sum_{i=1}^k(u'_\smy)_{[i]})_+\\
    &=\Li{4}(u',y)~.
\end{align*}

\textbf{Case 2:} $y\not\in S$:
In this case, we have $u_y \leq u'_y = 0$.
Moreover, $\sum_{i=1}^k (u_\smy)_{[i]} = \sum_{j \in S} u_j = \sum_{j \in S} u'_j = \sum_{i=1}^k (u'_\smy)_{[i]}$, as $S \in \Tk(u_\smy) \cap \Tk(u'_\smy)$. 
\begin{align*}
    \Li{4}(u,y)&=(1-u_y+\frac{1}{k}\sum_{i=1}^k(u_\smy)_{[i]})_+\\
    &\geq (1- 0 +\frac{1}{k}\sum_{i=1}^k(u_\smy)_{[i]})_+\\
    &=(1- u'_y +\frac{1}{k}\sum_{i=1}^k(u'_{\setminus y})_{[i]})_+\\
    &=\Li{4}(u',y)~.
\end{align*}

Therefore, for all $y$, we have $\Li{4}(u,y)\geq \Li{4}(u',y)$.
Thus, $\Rlo_4$ is representative.
\end{proof}

Using $\Rlo_4$ as a starting point, we now proceed to show $\U_4 := \Rlo_4\cap \Rhi_4$ is a representative set for $\Li{4}$.

\begin{theorem}\label{thm:U4-representative}
The set $\U_4 := \Rlo_4\cap \Rhi_4$ is a representative set for $\Li{4}$.
\end{theorem}
\begin{proof}
Since $\Rlo_4$ is representative by Lemma~\ref{lem:Rlo4-representative}, consider $u\in \Rlo_4$. 
Moreover, if $u\notin \Rhi_4$, construct $u'\in \reals^n_+$ such that
\[u'_y=\begin{cases}
1+\frac{1}{k}\sum_{i=1}^k(u_{\backslash y})_{[i]}&u_y>1+\frac{1}{k}\sum_{i=1}^k(u_{\backslash y})_{[i]}\\
u_y& u_y\leq1+\frac{1}{k}\sum_{i=1}^k(u_{\backslash y})_{[i]}
\end{cases}~.~\]

Observe that $u'\in \Rlo_4\cap \Rhi_4$ by construction and $\forall y\in\Y, u_y \geq u'_y$.\\
Since $u \not \in \Rhi_4$, there is a $y \in \Y$ such that $u_y>1+\frac{1}{k}\sum_{i=1}^k(u_{\backslash y})_{[i]}$; we can equivalently write
\begin{equation}\label{eq:outside-U4}
    u_y=\left(1+\frac{1}{k}\sum_{i=1}^k(u_{\backslash y})_{[i]}\right)+\epsilon,
\end{equation}
for some $\epsilon >0$.
We now proceed in two cases: considering the ground truth $y' =y$ and $y' \neq y$.

\textbf{Case 1:} Suppose $y$ is the ground truth label:
\begin{align*}
    \Li{4}(u,y)&=(1-u_y+\frac{1}{k}\sum_{i=1}^k(u_{\backslash y})_{[i]})_+&\\
    &=\left(1-(1+\frac{1}{k}\sum_{i=1}^k(u_{\backslash y})_{[i]}+\epsilon)+\frac{1}{k}\sum_{i=1}^k(u_{\backslash y})_{[i]}\right)_+&\\
    &=(-\epsilon)_+&\text{ where $\epsilon >0\implies-\epsilon<0$}\\
    &=0~.~&
\end{align*}
As $u'_y$ is of the same form of eq.~\eqref{eq:outside-U4} with $\epsilon = 0$, we observe equality as $(\epsilon)_+ = (0)_+ = 0$.
Therefore, $\Li{4}(u,y)=0=\Li{4}(u',y)$, and $\Li{4}(u,y)\geq \Li{4}(u',y)$ immediately.
\textbf{Case 2:} Let $j \neq y$ be the ground truth label.\\
\begin{align*}
    \Li{4}(u,j)&=(1-u_j+\frac{1}{k}\sum_{i=1}^k(u_{\backslash j})_{[i]})_+~.
\end{align*}
By the case, we have $u_j=u'_j$.
\begin{align*}
    \Li{4}(u,j)&=(1-u_j+\frac{1}{k}\sum_{i=1}^k(u_{\backslash j})_{[i]})_+&\\
    &=(1-u'_j+\frac{1}{k}\sum_{i=1}^k(u_{\backslash j})_{[i]})_+&\\
    &\geq(1-u'_j+\frac{1}{k}\sum_{i=1}^k(u'_{\backslash j})_{[i]})_+ & \text{ as $u\geq u'$ element-wise}\\
    &=\Li{4}(u',j)~.~
\end{align*}
Therefore, $\Li{4}(u,j)\geq \Li{4}(u',j)$.
Thus, we conclude 
\[\Li{4}(u,y)\geq \Li{4}(u',y)\quad\forall y\in\Y,\]
and therefore $\U_4 := \Rhi_4\cap \Rlo_4$ is a bounded, infinite, representative set for $\Li{4}$.
\end{proof}
\subsection{Characterizing Affineness of $\Li{4}$}
For $u\in \U_4$, we know that $\forall y\in\Y$
\begin{align}
    u_y&\leq 1+\frac{1}{k}\sum_{i=1}^k(u_{\backslash y})_{[i]} & \text{ as $u\in \Rhi_4$}\\
    \implies 0&\leq1-u_y+\frac{1}{k}\sum_{i=1}^k(u_{\backslash y})_{[i]}~. & 
\end{align}
Therefore, for all ground truth labels $y \in \Y$ and $u \in \U_4$, we have 
\begin{align*}
    \Li{4}|_{\U_4}(u,y)&=(1-u_j+\frac{1}{k}\sum_{i=1}^k(u_\smy)_{[i]})_+\\
    &=1-u_y+\frac{1}{k}\sum_{i=1}^k(u_\smy)_{[i]}
\end{align*}
is an equivalent way to write the loss $\Li{4}$ when restricting the domain to $\U_4$.
When restricting to $u \in \U_4$, we may denote $\Li 4(u,y) = \Li 4|_{\U_4}(u,y)$ for brevity and drop the positive part operator.

Now consider a set $T \subseteq \Y$ such that $|T|\leq k$.
Let us define the region \[A_4^T=\left\{u\in\U_4 \mid \begin{cases}
 0\leq u_y\leq 1+\frac{1}{k}\sum_{i\in T,i\neq y}u_i& y\in T\\
 u_y=0&y\notin T
\end{cases}\right\}\]
We claim, for any $y \in \Y$, the function $u \mapsto \Li{4}(u, y)$ is affine on $A^T$, and note that $A^T\subseteq \U_4$ for all $T$ by construction.

\begin{lemma}
For all $y\in \Y$ and set $T \subseteq \Y$ such that $|T| \leq k$, the function $u \mapsto \E_p \Li{4}(u,\cdot)$ defined on $\U_4$ is affine on $A^T$.
\end{lemma}
\begin{proof}
Fix $y \in \Y$ and $T \subseteq \Y$ such that $|T| \leq k$.
Note that for $u\in \U_4$,
\[\Li{4}(u,y)=1-u_y+\frac{1}{k}\sum_{i=1}^k(u_{\backslash y})_{[i]}\]
The first two terms of this loss are linear in $u$; therefore $\frac{1}{k}\sum_{i=1}^k(u_{\smy})_{[i]}$ is the only term with non-linearity.
Moreover, this term results from the ordering of the top $k$ elements of $u_{\smy}$. 
Given that $|T|\leq k$ and all elements of $u\notin T$ are $0$, we have that $T \in \Tk(u_\smy)$
Therefore, $u \mapsto Li 4(u,y)$ will be linear for $u\in A^T$.
\end{proof}

This result yields affine regions over which $u \mapsto \Li 4(u,y)$ is affine for each $y\in \Y$.
The vertices of these affine regions yield a finite representative set for $\Li 4$.

\subsection{Constructing a Finite Representative Set for $\Li{4}$}
Each set $T$ has a finite set of vertices according to the two inequalities shown in the definition of $A^T$ above. 
Since $|T|\leq k,$ there are a finite number of possible sets $T$ ( ${n \choose 0} + {n\choose 1}+{n\choose 2}+...+{n\choose k} =2^n$ possible sets in particular). 
Therefore, \[\bigcup_{T \subseteq \Y, \hspace{0.5mm}|T|\leq k}A^T = \U_4\] has a finite number of vertices.

According to the boundaries of the halfspaces defining $A^T$, the vertices of $A^T$ must be such points $u$ such that $u_y=0$ or $u_y=1+\frac{1}{k}\sum_{i\in T,i\neq y}u_i$ for each $y \in \Y$. 
Consider when $u_y=1+\frac{1}{k}\sum_{i\in T,i\neq y}u_i,$ which we will refer to as the ``bumped up'' value of $u_y$.
\begin{theorem}
Fix $T \subseteq \Y$ such that $1 \leq |T| \leq k$.
For a vertex $u$ in the region $A^T$ with $y\in \Y$ such that $u_y=1+\frac{1}{k}\sum_{i\in T, i\neq y}u_i,$ then $\forall i\in\Y$ such that $u_i \neq 0\implies u_i=u_y.$
\end{theorem}
\begin{proof}
Let $u_y=1+\frac{1}{k}\sum_{i\in T,i\neq y}u_i$ and $j\in T \implies u_j=1+\frac{1}{k}\sum_{i\in T, i\neq j}u_j$. We will show that $u_y=u_j,$ so all ``bumped up'' elements of $u$ must be equal to one another:
\begin{align*}
    u_y &= 1+\frac{1}{k}\sum_{i\in T, i\neq y}u_i\\
    u_y &= 1+\frac{1}{k}\sum_{i\in T, i\neq j}u_i+\frac{1}{k}u_j-\frac{1}{k}u_y\\
    \frac{k+1}{k}u_y &= u_j+\frac{1}{k}u_j\\
    \frac{k+1}{k}u_y &= \frac{k+1}{k}u_j\\
    u_y &= u_j
\end{align*}
Therefore, for any two arbitrary elements $y,j \in T, u_j=u_y$.
\end{proof}
Therefore, the closed for over vertices $u$ of the region $A^T$ are as follows:
\begin{align*}
    u_y&=1+\frac{1}{k}\sum_{i\in T, i\neq y}u_i\\
    u_y&=1+\frac{1}{k}\sum_{i\in T}u_i-\frac{1}{k}u_y\\
    u_y&=1+\frac{1}{k}|T|u_y-\frac{1}{k}u_y\\
    u_y&=\frac{k}{k+1-|T|}~.
\end{align*}
Thus, all of the vertices of each $A^T$ occur at $u \in \U_4$ such that $u_y=0$ or $\frac{k}{k+1-|T|}$ for all $y\in\Y$ where $|T| \leq k$ is the number of non-zero elements of $u\in A^T$.
consider the set of subsets $\T = \{T \subseteq \{1, \ldots, n\} \mid |T| \leq k\}$ and finite report set vertices of the $A^T$ sets by $\Ri{4} := \{ \frac{k}{k+1 - |T|} \ones_T \mid T \in \T\}$.

\subsection{Characterizing the Loss Embedded by $\Li{4}$}

We reparameterize the vertices of the $A^T$ sets by their defining set $T$ by the bijection $\Phi : T \mapsto \frac{k}{k+1 - |T|} \ones_T$.
We define the reparameterization $\hat \ell_4$ such that $\Li 4(\Phi(T),y) = \Li 4|_{\U_4}(\Phi(T),y) = \hat \ell_4(T, y)$ for all $u \in \T$.

We know that $\Li{4}$ embeds $\Li{4}|_{\U_4}$, and therefore also embeds 
\begin{align}
  \hat \ell_4(T,y) &=
\begin{cases}
 1-\frac{k}{k+1-|T|} +\frac{1}{k}(|T|-1)\frac{k}{k+1-|T|} & y\in T\\
 1+\frac{1}{k}|T|\frac{k}{k+1-|T|} & y\notin T
\end{cases} \nonumber \\
&=
\begin{cases}
 0  & y\in T\\
 \frac{k+1}{k+1-|T|} & y\notin T
\end{cases}~.
\end{align}
In a slight abuse of notation, for a set $T\subset[n]$ and $p\in\Delta_{\Y},$ let $\sigma_T(p)=\sum_{i\in T}p_i$. 
Therefore, the expected value of $\hat \ell_4$
is
\[\E_p \hat \ell_4(T,\cdot)=\sum_{y\in T}p_y(0)+\sum_{y\notin T}p_y(\frac{k+1}{k+1-|T|})=(1-\sigma_T(p))(\frac{k+1}{k+1-|T|})~.~\]

Now, suppose we have some set $T \in \T$ as defined above; we will analyze $\prop{\hat \ell_4}$ to determine the necessary probability $p_i$ some in $i \in \Y$ so that $i \in T \in \prop{\hat \ell_4}(p)$. 
In other words, if we have some set $T$ of labels corresponding to scores in $u$ of $\frac{k+1}{k+1-|T|},$ then we will bump the score of some label, $z\notin T$, up to $\frac{k+1}{k+1-|T|-1}$ (changing all of the non-zero scores in $u$ to this value as well) if it surpasses a particular probability threshold. 
We will find this probability boundary below by seeing what probability $z\in\Y$ must achieve in order to meet or lower the expected loss:
\[\E_{p}\hat \ell_4(T,\cdot)\geq \E_{p}\hat \ell_4(T\cup \{z\},\cdot)~.\]

By doing so, we are determining the probability of $p_z$ such that, for a fixed $p\in \simplex$, we have $T\in \prop{\hat \ell_4}(p)\implies T\cup \{z\}\in\prop{\hat\ell_4}(p)$.\\
This boundary is given:
\begin{align}
    \E_{p} \hat \ell_4(T,\cdot) &= \E_{p}\hat \ell_4(T\cup \{z\},\cdot) \nonumber \\
    (1-\sigma_T(p))\frac{k+1}{k+1-|T|} &= (1-\sigma_T(p)-p_z)\frac{k+1}{k+1-|T|-1} \nonumber \\
    (k-|T|)(1-\sigma_T(p)) &= (k+1-|T|)(1-\sigma_T(p)-p_z) \nonumber \\
    0 &= 1-\sigma_T(p)-p_z(k+1-|T|) \nonumber \\
    p_z &= \frac{1-\sigma_T(p)}{k+1-|T|}\label{eq:L4-bound}
\end{align}

Therefore, to add the element $z$ to the set $T$
\[p_z\geq \frac{1-\sigma_T(p)}{k+1-|T|}.\]
Iteratively adding elements such that the above boundary holds will be necessary and sufficient to form an optimal set $M^*\subset[n]$ of labels that minimizes $\E_p \hat \ell_4(M^*, \cdot)$, and equivalently, $\Phi(M^*)$ minimizes $\Li{4}.$

\begin{theorem}
Consider $\gamma_4 := \prop{\hat \ell_4}$.
Fix $p \in \simplex$ and $T \in \T$ be such that $T$ is the top-$|T|$ elements of $p$ with $|T|\leq k-1$.
Consider $z\in[n] \setminus T$ such that $p_z\geq\frac{1-\sigma_T(p)}{k+1-|T|}$ and $p_i\leq p_z$ for all $i \in[n]\backslash T$. 
Then $z$ must be an element of $M^*$ for some $M^* \in \gamma_4(p)$. 
\end{theorem}

\begin{proof}
For intuition, $T$ is a set of labels at least as likely as label $z$.
Observe that there is an $M \in \gamma_4(p)$ such that $T \subseteq M$ since $T$ is composed of the top-$|T|$ elements of $p$, and replacing any $t \in T$ with $t' \not \in T$ cannot decrease expected loss as the denominator stays the same and $T$ is composed of the top-$|T|$ elements of $p$.

It is not necessarily the case that $M = M^*$ as we may have $|\gamma_4(p)| >1$ and the top-$k$ elements of the property value are ambiguous.
If $|\gamma_4(p)| = 1$, however, then we must have $M = M^*$.
Suppose $z\notin M$ (otherwise this proof is trivial), we have two cases:\\
\textbf{Case 1:} $T \subsetneq M$, e.g., $\exists z'\in M$ such that $z'\notin T$.
\begin{align*}
    \mathbb{E}_{p}\hat \ell_4(M,\cdot)&=(k+1)\frac{1-p_{z'}-\sigma_{M\backslash\{z'\}}(p)}{k+1-|M|}\\
    &\geq(k+1)\frac{1-p_{z}-\sigma_{M\backslash\{z'\}}(p)}{k+1-|M|} \quad \text{from $p_z'\leq p_z$~.}
\end{align*}

Therefore, $\exists M^*$ which is optimal such that $z\in M^*=M\backslash \{z'\}\cup \{z\}$.\\
\textbf{Case 2:} $T=M$. 
By the assumptions and choice of $z$,
\[p_z\geq \frac{1-\sigma_T(p)}{k+1-|T|}~.\]
Therefore, using the bound from eq.~\eqref{eq:L4-bound}, we have
\[\E_{p}\hat \ell_4(T\cup \{z\},\cdot)\leq \E_{p}\hat \ell_4(T,\cdot)=\E_{p}\hat \ell_4(M,\cdot)~.\]
As $M$ is optimal, $M^*=T\cup\{z\}$ is also optimal.

From both cases above, we can conclude $z\in M^*$ for some optimal set $M^*$.
\end{proof}

This will enable us to characterize $\gamma_4$ in Theorem~\ref{thm:L4-monotonic}.
However, we first need the following Lemma.
\begin{lemma}\label{lem:L4-monotonic-fracs}
For $a,c\in\reals_+$ and $b,d\in\reals_{++}$ with $b>d$,
\[\frac{c}{d}<\frac{a}{b}\implies \frac{a+c}{b+d}<\frac{a}{b}.\]
\end{lemma}
\begin{proof}
\begin{align*}
    \frac{c}{d}&<\frac{a}{b}\\
    \iff cb &<ad\\
    \iff ab+cb&<ab+ad\\
    \iff \frac{a+c}{b+d}&<\frac{a}{b}
\end{align*}
\end{proof}

Now we obtain the following result to characterize $\gamma_4$.
\begin{theorem}\label{thm:L4-monotonic}
Fix $p \in \simplex$, and consider any $T\subset[n]$ which minimizes $\E_p \hat \ell_4(T, \cdot)$, i.e., $T \in \gamma_4(p)$.
Then for all $z\in T$, we have $p_z\geq \frac{1-\sigma_T(p)}{k+1-|T|}$. 
\end{theorem}
\begin{proof}
We will show the contrapositive.
For $T \in \gamma_4(p)$.
Suppose there was a $z \in T$ such that $p_z < \frac{1 - \sigma_T(p)}{k + 1 - T}$.
We will contradict optimality of $T$ by showing $\E_p \hat \ell_4(T \setminus \{z\}, \cdot) < \E_p \hat \ell_4(T, \cdot)$.
Denote $M := T \setminus \{z\}$.

Note, that if we let $c=p_z \in\reals_{+},$  $d=1 \in\reals_{++},$ $a = (1-\sigma_M(p))\in\reals_{+},$ and $b = (k+1-|T|)\in\reals_{++},$ then we have $    p_z < \frac{1-\sigma_T(p)}{k+1-|T|} \iff \frac{c}{d} < \frac{a}{b}$.
Thus, we can apply Lemma~\ref{lem:L4-monotonic-fracs} to observe
\begin{align*}
    \frac{a+c}{b+d} &< \frac{a}{b}\\
    \frac{1-\sigma_T(p)+p_z}{k-(|T| - 1) + 1} &< \frac{1-\sigma_T(p)}{k+1-|T|} \\
    \implies (k+1)\frac{1-\sigma_M(p)}{k-|M| + 1} &< (k+1)\frac{1-\sigma_T(p)}{k+1-|T|} \\
    \implies \E_p \hat \ell_4(M,\cdot) &< \E_p \hat \ell_4(T,\cdot)
\end{align*}
Therefore, the expected loss on $M\subset[n]$ is strictly lower than on $T$; thus, $T \not \in \gamma_4(p)$.
Thus for any $T \in \gamma_4(p)$, we must have $p_z \geq \frac{1 - \sigma_T(p)}{k+1 - |T|}$ for all $z\in T$.
\end{proof}

By Theorem~\ref{thm:L4-monotonic}, we can conclude that iteratively adding elements $z\in[n]$ (in increasing order of corresponding probability) such that \[p_z\geq \frac{1-\sigma_T(p)}{k+1-|T|}\] to a set $T\subseteq[n],$ that is initially the empty set, is necessary and sufficient to form the optimal set $M^*\subset [n]$ that minimizes $\mathbb{E}_{p}\hat \ell_4(M^*,\cdot).$
That is, $\prop{\Li{4}|_{\U_4}}$ can be computed by implementing a greedy algorithm.
\subsection{A sketch of $\prop{\Li{4}}$}
Let $T\subset [n]$ where the elements of $T$ have been iteratively added in decreasing order of probability so long as $|T|\leq k$ and the probability of the added item meets the boundary condition defined above. 
Suppose $|T| \leq k-1$, then from our derivation of the probability needed to add an element $z$ to $T$, we can rewrite the boundary condition as adding an element $u_{[j]}$ with $j\in\{1,2,...,k\}$, so long as
\[p_{[j]}\geq\frac{1-\sum_{i=1}^{j-1}p_{[i]}}{k+2-j}.\]
We can rewrite rewrite the above as
\begin{align}
(k+1-j)p_{[j]} &\geq 1-\sigma_{j}(p)~.
\end{align}
Let $j_1\in[n]$ be the largest $j$ such that 
\begin{align}\label{eq:j1}
(k+1-j)p_{[j]} &> 1-\sigma_j(p)~.
\end{align}
Let $j_2 \in [n]$ be the largest $j$ such that 
\[(k+1-j)p_{[j]}\geq 1-\sigma_j(p).\]
\begin{lemma}
For all $p \in \simplex$ and $j_1$ as in eq.~\eqref{eq:j1}, we have $j_1\leq k$.
\end{lemma}
\begin{proof}
Suppose for the sake of contradiction that $j_1> k$, and therefore $j_1\geq k+1.$
Then, we have
\begin{align*}
    (k+1-j_1)p_{[j_1]}&\leq 0
\end{align*}
By definition of $j_1$,
\begin{align*}
(k+1-j_1)p_{[j_1]}>1-\sigma_{j_1}(p)
\implies    0>&1-\sigma_{j_1}(p)\\
    \sigma_{j_1}(p)>&1~.
\end{align*}
However, this contradicts that $\sigma_{j_1}(p)\leq\sum_{i=1}^np_{[i]}=1$, as $p \in \simplex$.
Thus, we conclude that $j_1\leq k$.
\end{proof}
Note that if for some $j\in[n]$,
\[(k+1-j)p_{[j]} = 1-\sigma_{j}(p),\]
then the expected loss will not change by ``bumping up'' the corresponding element in $u$. Therefore, we are indifferent to ``bumping up'' this element or not.\\
From the above definitions define two sets $H : \simplex \to 2^{[n]}$ and $I : \simplex \to 2^{[n]}$ as follows:
\begin{align*}
    H(p) &= \left\{i\in[n] \mid p_i\geq p_{[j_1]}\right\}\\
    I(p) &= \left\{i\in[n] \mid p_{[j_2]} \leq  p_i < p_{[j_1]}\right\}
\end{align*}
Note, that $T = H(p)\cup I(p)$ is a minimizing set of indices for $\E_p \hat \ell_4(T, \cdot)$ when we ``bump up'' exactly those corresponding elements in $H(p)$.
If $p$ is understood from context, then we simply denote $H(p) = H$, etc.

Intuitively, $H$ (``high'') is the set of elements that bumping up (including in the report set $T$) will result in a lower expected loss. 
$I$ (``indifferent'') is the set of elements that bumping up will not affect expected loss, meaning we are indifferent to bumping them up.

From these definitions, we can see that the set of all $H\cup I^*$ where $I^*\in P(I)$ (the power set $2^I)$ such that $|I^*|\leq k-j_1,$ will have an expected loss equal to the expected loss associated with the set $H\cup I$. 
And $H\cup I$ is the exact set constructed by iteratively adding elements according to the boundary condition defined above (and we established above that this is the strategy for forming an optimal report set when $|H\cup I|\leq k$). 
Therefore, the set of all $H\cup I^*$ where $I^*\in P(I)$ such that $|I^*|\leq k-j_1,$ will be representative.
In particular, there is an $I^* \in P(I)$ (e.g., $I^* = \emptyset$) such that $|H \cup I^*| \leq k$, so that $\Psi(T) \in \Ri{4}$, where $T = H \cup I^*$.

We can conclude that the property elicited by $\hat \ell_4$ is given
\[\prop{\Li{4}}(p)=\left\{\frac{k+1}{k+1-|T|}\ones_{T} \mid T=H\cup I^*, I^*\in P(I), |I^*| \leq k-j_1\right\} ~,\]
where $P(I)$ is the power set of set $I$, and $H$ and $I$ are functions of $p$.\\

\subsection{Characterizing Consistency of $\Li{4}$}
From this, we can conclude that $\hat\ell_4$ indirectly elicits top-$k$ when $j_1=k$ because in all other cases $\prop{\hat \ell_4}$ will return a set with cardinality greater than 1 which will require the breaking of ties. This breaking of ties is dependent on the link utilized, which in this case is the $\argmax$; however, as established we would be breaking ties between sets that result in the same expected loss of $\Li{4}.$ This means that we would be breaking ties arbitrarily. The only case in which this does not occur is when we are not indifferent between bumping up any elements $u_i, u_j$ where  $i,j\in [n], i\neq j.$ This occurs when $j_1=k,$ resulting in 
\begin{align*}
    (k+1-k)p_{[k]}&>1-\sigma_k \qquad\text{by definition of $j_1$}\\
    p_{[k]}&>1-\sigma_k
\end{align*}
Therefore by Lemma \ref{lemma:restricted-consistency}, we know that $\Li{4}$ is guaranteed consistency with top-$k$ when $p_{[k]} > 1-\sigma_k$.

\ellfourinconsistent*

\section{Additional Derivations for $\Lk$}\label{appendix:new-surrogate}

\subsection{Proof of Lemma \ref{lemma:Gamma}} \label{appendix:Gamma-lemma-proof}
As $\Lk$ is a proper polyhedral function, we know that it attains its infimum~\citep[Corollary 19.3.1]{rockafellar1997convex}, and thus $\Gamma$ is well-defined on $\simplex$.
Let $G(u) = (- \risk{\lk})^*(u)$ and $I_{\simplex}$ be the convex indicator function that is 0 on $\simplex$ and $\infty$ on $\reals^n \setminus \simplex$.
Then, $G^*(p) = -\risk{\lk}(p) = \sumk(p) + I_{\simplex}(p) - 1$.

\begin{lemma}
	\label{lemma:G}
	$\Gamma(p) = \partial G^*(p)$.
\end{lemma}
\begin{proof}
    As $\Lk$ is a proper convex function, \citet[Theorem 23.5]{rockafellar1997convex} yields
    \begin{align*}
        u \in \partial G^*(p) &\iff G(u) + G^*(p) = \inprod{u}{p} & \text{\citet[Theorem 23.5]{rockafellar1997convex}} \\
        &\iff \inprod{p}{\Lk(u,\cdot)} = -G^*(p) & \text{\citet[Theorem 4]{finocchiaro2022embedding}} \\
        &\iff \inprod{p}{\Lk(u,\cdot)} = \risk \Lk(p) & \text{\citet[Theorem 4]{finocchiaro2022embedding}} \\
        &\iff u \in \argmin_{u'} \inprod{p}{\Lk(u', \cdot)} = \Gamma(p)~.
        &&\qedhere
    \end{align*}
\end{proof}

Therefore, we just need to characterize the subgradients of $G^*$. 
As $\Lk$ is polyhedral, we know that is is the pointwise maximum of a finite number of affine (and therefore convex) functions.
This enables us to use a result from~\citet{hiriart2012fundamentals} to rewrite the subdifferential of $G^*$ in order to characterize $\Gamma(p)$ for all $p \in \relint(\simplex)$.

\begin{theorem}[\citet{hiriart2012fundamentals}{[D.4.3.2]}]
\label{thm:subgradient-max}
Let $f_1, ... f_m$ be convex functions from $\reals^n \to \reals$.
Then,
\[\partial \left(\max_i f_i(x)\right) = \hull\left\{\cup_i \partial f_i(x) \mid i \in \argmax_j f_j(x) \right\} ~.\]
\end{theorem}

\begin{lemma}
	\label{lem:S-subgradient}
	For all $p \in \simplex$, we have $\partial \sigma_k(p) = \hull\{\topkvec(p)\}$. 
\end{lemma}

\begin{proof} 
Let $f_t(p) = \inprod{t}{p}$ for each $t \in \T$.
By affineness, $\partial f_t(p) = \{t\}$. 
Now, recalling the definition of $\sigma_k$, we can write
\begin{align*}
	\partial \sigma_k(p) 
	&= \partial \left( \max_{t \in \T} \inprod{t}{p} \right)& \\
	&= \partial \left( \max_{t \in \T} f_t(p) \right) & \\
	&= \hull \left\{ \cup_{t} \left(\partial f_t(p)\right) \mid t \in \argmax_{t'} f_{t'}(p) \right\} & \text{Theorem \ref{thm:subgradient-max}}\\
	&= \hull \left\{ \cup_{t} \left\{t\right\} \mid t \in \argmax_{t'} f_{t'}(p) \right\} & \\
	&= \hull \left\{\argmax_{t} f_{t}(p) \right\} & \\
	&= \hull \left\{\argmax_{t} \inprod{t}{p} \right\}&  \\
	&= \hull \left\{\topkvec(p) \right\} ~.
	&&\qedhere
\end{align*}
\end{proof}

\begin{lemma}
	\label{lemma:convexindicator-subgradient}
	For all $p$ on the relative boundary of $\simplex$, (that is, $\simplex \setminus \relint(\simplex)$), 
	\[ \partial I_{\simplex}(p) = \cup_{\alpha \in\R} \{\alpha \ones\} - \cone\{\ones_{i} | p_i = 0 \}~.\] 
	Moreover, $\partial I_\simplex(p) = \vec 0$ for all $p \in \relint(\simplex)$.
\end{lemma}

\begin{proof}
We can define the simplex as the set of points $p \in \reals^n$ that satisfies the constraints $\inprod{p}{\ones} = 1$, and $\inprod{p}{\ones_i} \geq 0$ for all $1 \leq i \leq n$. 
Let $I_0$ be the convex indicator of the first constraint, such that $I_0(p) = 0$ when $\inprod{p}{\ones} = 1$, and $I_0(p) = \infty$ otherwise. 
Similarly, let $I_i$ be the convex indicators such that $I_i(p) = 0$ if $\inprod{p}{\ones_i} \geq 0$, and $I_i(p) = \infty$ otherwise.
A point $p \in \reals^n$ will be in $\simplex$ precisely when all $n + 1$ constraints are satisfied, which is exactly when all the indicators are 0. 
Therefore, we can rewrite
\[I_\simplex(p) = \sum_{i=0}^n I_i(p) ~. \]

For any $p \in \simplex$, we have $\partial I_0(p) = \left\{ \alpha_0 \ones | \alpha_0 \in \reals \right\}$.
For $1 \leq i \leq n$, $\partial I_i(p) = \{\vec 0\}$ if $p_i > 0$, and $\partial I_i(p) = \{ - \alpha_i \ones_i \mid \alpha_i > 0\}$ if $p_i = 0$.

The subgradient of a sum of convex functions is the Minkowski sum of their individual subgradients \citep[Theorem 23.8]{rockafellar1997convex}. Now, we observe, 
\begin{align*}
    \partial I_{\simplex}(p) 
    &= \partial \left( \sum_{i=0}^n I_i(p) \right) & \\
    &= \sum_{i=0}^n \partial I_i(p)  & \\
    &= \partial I_0(p) + \sum_{i=1}^n \partial I_i(p)  & \\
    &= \cup_{\alpha_0 \in \reals}\{\alpha_0 \ones\} + \sum_{i=1}^n \{- \alpha_i \ones_{i} | p_i = 0, \alpha_i \geq 0 \}  & \\
    &= \cup_{\alpha \in\R} \{\alpha \ones\} - \cone\{\ones_{i} | p_i = 0 \} ~. 
    &&\qedhere
\end{align*}

\end{proof}

\begin{lemma}
	$\Gamma(p) = \hull\{\topkvec(p)\} -\cone\{\ones_{i} | p_i = 0 \} + \cup_{\alpha \in\R} \{\alpha \ones\}$.
\end{lemma}

\begin{proof}
	\begin{align*}
		\Gamma(p) 
		&= \partial G^*(p)  & \text{Lemma~\ref{lemma:G}}\qquad\\
		&= \partial \left(\sumk(p) + I_\simplex(p) - 1 \right) &\\
		&= \partial \sumk(p) + \partial I_\simplex(p) - 0 &\\
		&= \hull\{\topkvec(p)\} -\cone\{\ones_{i} | p_i = 0 \} + \cup_{\alpha \in\R} \{\alpha \ones\}~. & \text{Lemma \ref{lem:S-subgradient} and Lemma \ref{lemma:convexindicator-subgradient}}\qquad
		&\qedhere
	\end{align*}
\end{proof}

\subsection{Equivalence of Equations \ref{eq:new-surrogate} and \ref{eq:neq-surrogate-max-form} }
\label{subsec:new-surrogate-max-form}

\begin{lemma}
\[\Lk (u, y) = \max_{1\leq m \leq n} \left\{ \tfrac {\sigma_m(u)}{m} +  \left(1 - \tfrac k m\right)_+ \right\} - u_y ~.\]
\end{lemma}

\begin{proof}
By Equation \ref{eq:new-surrogate}, 
\[ \Lk (u, y) = \sup_{p \in \simplex}\left(\inprod{p}{u} - \sumk(p)\right) + 1 - u_y ~.\]

Without loss of generality, we may assume $u$ is sorted. 
Since $\sumk(p)$ is not order dependent, and $\inprod{p}{u}$ will be maximized when the elements of $p$ have the same ordering as the elements of $u$, we can assume $p$ is sorted as well.
Let $\sort(\simplex)$ denote the subset of vectors $p \in \simplex$ that are sorted.
The loss then simplifies to

\begin{align*}
    \Lk(u,y)
    &= \sup_{p \in \sort(\simplex)}\left( \sum_{i=1}^k p_i (u_i - 1) + \sum_{i=k+1}^n p_i u_i\right) + 1 - u_y ~. \\
\intertext{
Let $v$ be the vector such that for $i \leq k$, $v_i = u_i - 1$, and for $i > k$, $v_i = u_i$. 
We can then reduce to
}
    &= \sup_{p \in \sort(\simplex)}\left( \sum_{i=1}^n p_i v_i\right) + 1 - u_y \\
    &= \sup_{p \in \sort(\simplex)}\inprod{p}{v} + 1 - u_y ~.
\end{align*}

We claim that, for any fixed $v$, there exists a $p \in \argsup_{p' \in \simplex} \inprod{p'}{v}$ such that $p = \ones_M / |M|$ for some set $M \subseteq [n]$.

We proceed by contradiction.
Assume that there is no $p$ that is exactly $\frac{1}{m}$ on $m$ indices 
that achieves the supremum.
Let $U = \argsup_{p \in \simplex} \inprod{p}{v} \subseteq \simplex$ be the set of (sorted) distributions that do achieve the supremum.
Since $\simplex$ is compact and $\inprod{p}{v}$ is linear, $U$ is nonempty.
By assumption, for every $q \in U$, there must be some index $m$ such that $q_1 = q_m > q_{m+1} > 0$. 
Choose any $q$ with the maximal such $m$.
Let $\mu = \frac{1}{m} \sum_{i=1}^m v_i$ be the average of the first $m$ elements of $v$. 
Then, we have
\begin{align*}
    \inprod{q}{v} &= q_m \mu + \sum_{i=m+1}^n q_i v_i ~.
\end{align*}
If $ m \mu > \sum_{i>m} v_i q_i$, we can choose a sufficiently small $\epsilon > 0$ and set $q'_i = q_1 - \epsilon\frac{1 - m \mu}{m}$ for $i \leq m$ and $q'_i = (1 + \epsilon) q_i$ for $i > m$ to get a new distribution $q' \in \simplex$.
Using this $q'$ instead of $q$ increases $\inprod{q}{v}$, so $q \not \in U$, a contradiction.
If instead $ m \mu < \sum_{i>m} v_i q_i$, we can instead choose a sufficiently large $\epsilon < 0$, and achieve the same result.
If instead $ m \mu = \sum_{i>m} v_i q_i$, we can choose $\epsilon$ such that $q'_m = q'_{m+1}$, so we did not choose the $q$ with the maximal $m$, also a contradiction.

Therefore, there is some sorted $p$ that is $\frac{1}{m}$ on exactly $m$ indicies that achieves the supremum.
We therefore need only consider this set of distributions.
Plugging this into the original equation, we get

\begin{align*}
    L_k(u,y)
    &= \sup_{p \in \simplex}\left( \sum_{i=1}^k p_i (u_i - 1) + \sum_{i=k+1}^n p_i u_i\right) + 1 - u_y\\
    &= \max_{1\leq m \leq n} \left( \sum_{i=1}^k \frac{1}{m}(u_i - 1)\right) + 1 - u_y\\
    &= \max_{1\leq m \leq n} \left\{ \tfrac {\sigma_m(u)}{m} + \left(1 - \tfrac k m\right)_+ \right\} - u_y~.
    \qedhere
\end{align*}
\end{proof}

\end{document}